\documentclass{article}

\usepackage[utf8]{inputenc} %
\usepackage[T1]{fontenc}    %
\usepackage{url}            %
\usepackage{booktabs}       %
\usepackage{amsfonts}       %
\usepackage{natbib}
\usepackage{algorithmic}
\usepackage{algorithm}
\usepackage{array}
\usepackage{mathtools}
\usepackage{graphicx}

\usepackage{microtype}
\usepackage{subfigure}

\usepackage[accepted]{icml2022}

\usepackage{hyperref}

\usepackage{shortcuts}

\usepackage[capitalize, noabbrev]{cleveref}

\usepackage{amsmath}
\usepackage{amssymb}
\usepackage{mathtools}
\usepackage{amsthm}
\usepackage{thmtools} %

\theoremstyle{plain}
\newtheorem{theorem}{Theorem}[section]

\newtheorem{lemma}[theorem]{Lemma}

\theoremstyle{definition}
\newtheorem{definition}[theorem]{Definition}
\newtheorem{assumption}[theorem]{Assumption}
\crefname{assumption}{Assumption}{Assumptions}
\theoremstyle{remark}
\newtheorem{remark}[theorem]{Remark}

\usepackage[textsize=tiny]{todonotes}

\usepackage{nicefrac}       %
\usepackage{microtype}      %
\usepackage{enumitem}

\title{Learning Bellman Complete Representations for Offline Policy Evaluation}

\begin{document}

\twocolumn[
\icmltitle{Learning Bellman Complete Representations for Offline Policy Evaluation}

\icmlsetsymbol{equal}{*}

\begin{icmlauthorlist}
\icmlauthor{Jonathan D. Chang}{equal,cu}
\icmlauthor{Kaiwen Wang}{equal,cu}
\icmlauthor{Nathan Kallus}{cuorie}
\icmlauthor{Wen Sun}{cu}
\end{icmlauthorlist}

\icmlaffiliation{cu}{Computer Science, Cornell University, Ithaca, NY, USA}
\icmlaffiliation{cuorie}{Operations Research and Information Engineering, Cornell Tech, New York, NY, USA}

\icmlcorrespondingauthor{Jonathan D. Chang}{\href{https://jdchang1.github.io/}{\nolinkurl{https://jdchang1.github.io}}}
\icmlcorrespondingauthor{Kaiwen Wang}{\href{https://kaiwenw.github.io/}{\nolinkurl{https://kaiwenw.github.io}}}

\icmlkeywords{Representation Learning, Bellman Completeness, Off-Policy Evaluation, Reinforcement Learning, Offline RL}

\vskip 0.3in
]

\printAffiliationsAndNotice{\icmlEqualContribution} %

\begin{abstract}

We study representation learning for Offline Reinforcement Learning (RL), focusing on the important task of Offline Policy Evaluation (OPE).
Recent work shows that, in contrast to supervised learning, realizability of the Q-function is not enough for learning it.
Two sufficient conditions for sample-efficient OPE are Bellman completeness and coverage.
Prior work often assumes that representations satisfying these conditions are given, with results being mostly theoretical in nature.
In this work, we propose \texttt{BCRL}, which directly learns from data an approximately linear Bellman complete representation with good coverage.
With this learned representation, we perform OPE using Least Square Policy Evaluation (LSPE) with linear functions in our learned representation.
We present an end-to-end theoretical analysis, showing that our two-stage algorithm enjoys polynomial sample complexity provided \emph{some} representation in the rich class considered is linear Bellman complete.
Empirically, we extensively evaluate our algorithm on challenging, image-based continuous control tasks from the Deepmind Control Suite.
We show our representation enables better OPE compared to previous representation learning methods developed for off-policy RL (e.g., CURL, SPR).
\texttt{BCRL} achieves competitive OPE error with the state-of-the-art method Fitted Q-Evaluation (FQE), and beats FQE when evaluating beyond the initial state distribution.
Our ablations show that both linear Bellman complete and coverage components of our method are crucial.

\end{abstract}

\newcommand{\rewardVec}{\rho}
\newcommand{\Mub}{\nm{M\opt}_2}
\newcommand{\Dcal}{\mathcal{D}}
\newcommand{\Tcal}{\mathcal{T}}
\newcommand{\Scal}{\mathcal{S}}
\newcommand{\Acal}{\mathcal{A}}
\newcommand{\alg}{\texttt{BCRL}}

\vspace{-20pt}
\section{Introduction}\label{sec:intro}

Deep Reinforcement Learning (RL) has developed agents that solve complex sequential decision making tasks, achieving new state-of-the-art results and surpassing expert human performance. Despite these impressive results, these algorithms often require a prohibitively large number of online interactions to scale to higher dimensional inputs.

To address these sample complexity demands, a recent line of work \cite{oord2018cpc, anand2019stdim, mazoure2020infomax, nachum2021pretraining} has incorporated advances in unsupervised representation learning from the supervised learning literature into developing RL agents.
For example, CURL \cite{laskin2020curl} and SPR \cite{spr} utilize contrastive representation objectives as auxiliary losses within an existing RL framework.
These efforts are motivated by the tremendous success such self-supervised techniques have offered in computer vision, natural language processing, speech processing, and beyond. %
While these formulations have shown sample complexity improvements empirically, it remains an open question whether these approaches successfully address the unique challenges from RL that usually do not appear in supervised learning, such as exploration and exploitation, credit assignments, long horizon prediction, and distribution shift. 
In particular, recent work \citep{wang2021exponential,amortila2020variant,foster2021offline} has shown that realizability of the learning target in RL (namely, the Q-function) is insufficient to avoid exponential dependence on problem parameters.%
\looseness=-1

In this paper, we study representation learning for an important subtask of off-policy RL: offline policy evaluation (OPE). OPE is a critical component for any off-policy policy optimization approach (e.g., off-policy actor critic such as SAC, \citealp{haarnoja2018soft}, and DDPG, \citealp{lillicrap2016ddpg}). Moreover, OPE allows us to focus on issues arising from distribution shift and long horizon prediction. Specifically, we propose a new approach that leverages rich function approximation (e.g., deep neural networks) to learn a representation that is both Bellman complete and exploratory. A linear Bellman complete representation means that linear functions in the representation have zero inherent Bellman error \cite{antos2008fqi}, i.e., applying the Bellman operator on a linear function results in a new linear function. An exploratory representation means that the resulting feature covariance matrix formed by the offline dataset is well-conditioned. These two representational properties ensure that, under linear function approximation (i.e., the linear evaluation protocol, \citealp{grill2020byol}), classic least squares policy evaluation (LSPE) \cite{nedic2003least,duan2020minimax,ruosong2020statisticallimits} can achieve accurate policy evaluation. We provide an end-to-end analysis showing that our representation learning approach together with LSPE ensures near-optimal policy evaluation with polynomial sample complexity. \looseness=-1

Empirically, we extensively evaluate our method on image-based continuous control tasks from the Deepmind Control Suite. First, we compare against two representation learning approaches developed for off-policy RL: CURL \cite{laskin2020curl} and SPR \cite{spr}. These bear many similarities to contrastive learning techiniques for unsupervised representation learning: SimCLR \cite{chen2020simclr} and Bootstrap your own latent (BYOL, \citealp{grill2020byol}), respectively.
Under the linear evaluation protocol (i.e., LSPE with a linear function on top of the learned representation), our approach consistently outperforms these baselines. We observe that representations learned by CURL and SPR  sometimes even exhibit \emph{instability when evaluated using LSPE} (prediction error blows up when more iterations of LSPE is applied), while our approach is more stable. This comparison demonstrates that representation learning in offline RL is more subtle and using representation techniques developed from supervised learning settings may not result in the best performance for offline RL. 
Our ablations show that both linear Bellman completeness and coverage are crucial, as our method also blows up if one ingredient is missing. Finally, \texttt{BCRL} achieves state-of-the-art OPE error when compared with other OPE methods, and improves the state-of-the-art when evaluating beyond the initial state distribution.\footnote{{\scriptsize Code available at \href{https://github.com/CausalML/bcrl}{\nolinkurl{https://github.com/CausalML/bcrl}}}.}
\looseness=-1

\subsection{Related Work}

\textbf{Representation Learning in Offline RL: }

\textbf{From the theoretical side}, \citet{hao2021sparse} considers offline RL in sparse linear MDPs \cite{jin2020provably}. Learning with sparsity can be understood as feature selection. Their work has a much stronger coverage condition than ours: namely, given a representation class $\Phi$, they assume that any feature $\phi\in\Phi$ has global coverage under the offline data distribution, i.e., $\EE_{s,a\sim \nu}\phi(s,a)\phi(s,a)^\top$ is well conditioned where $\nu$ is offline data distribution. In our work, we only assume that \emph{there exists one $\phi^\star$} that has global coverage, thus strictly generalizing their coverage condition. \citet{uehara2021pessimistic} propose a general model-based offline RL approach that can perform representation learning for linear MDPs in the offline setting without global coverage. However, their algorithm is a version space approach and is not computationally efficient. Also, our linear Bellman completeness condition \emph{strictly generalizes} linear MDPs. \cite{ni2021learning} consider learning state-action embeddings from a known RKHS. We use general function approximation that can be more powerful than RKHS. Finally, \citet{par2008analysislinearfeature} identifies bellman completeness as a desirable condition for feature selection in RL when analyzing an equivalence between linear value-function approximation and linear model approximation. In our work, we  investigate how to learn bellman completeness representation, and also the role of coverage in our feature selection.

\textbf{From the empirical side}, \citet{nachum2021pretraining} evaluated a broad range of unsupervised objectives for learning pretrained representations from offline datasets for downstream Imitation Learning (IL), online RL, and offline RL. They found that the use of pretrained representations dramatically improved the downstream performance for policy learning algorithms. In this work, we aim to identify such an unsupervised representation objective and to extend the empirical evaluation to offline image-based continuous control tasks. \citet{nachum2021replearning} presents a provable contrastive representation learning objective for IL (derived from maximum likelihood loss), learning state representations from expert data to do imitation with behavior cloning. Our approach instead focuses on learning state-action representations for OPE. Finally, both \citet{song2016linearfeatureencoding} and \citet{chung2018twotimescale} present algorithms to learn representations suitable for linear value function approximation. \citet{song2016linearfeatureencoding} is most relevant to our work where they aim to learn bellman complete features. Both works, however, work in the online setting and do not consider coverage induced from the representation which we identify is important for accurate OPE.

\textbf{Representation Learning in Online RL:} 

\textbf{Theoretical works} on representation learning in online RL focus on learning representations to facilitate exploration from scratch. OLIVE \citep{jiang2017contextual}, BLin-UCB \citep{du2021bilinear}, FLAMBE \citep{agarwal2020flambe}, Moffle \citep{modi2021model}, and Rep-UCB \citep{uehara2021representation} propose approaches for representation learning in linear MDPs where they assume that there exists a feature $\phi^\star\in\Phi$ that admits the linear MDP structure for the ground truth transition. \citet{zhang2021provably} posits an even stronger assumption where every feature $\phi\in\Phi$ admits the linear MDP structure for the true transition. Note that we only assume that \emph{there exists one} $\phi^\star$ that admits linear Bellman completeness, which strictly generalizes linear MDPs. Hence, our representation learning setting is more general than prior theoretical works. Note that we study the offline setting while prior works operate in the online setting which has additional challenges from online exploration.

\textbf{On the empirical side}, there is a large body of works that adapt existing self-supervised learning techniques developed from computer vision and NLP to RL. For learning representations from image-based RL tasks, CPC \cite{oord2018cpc}, ST-DIM \cite{anand2019stdim}, DRIML \cite{mazoure2020infomax}, CURL \cite{laskin2020curl}, and SPR \cite{spr} learn representations by optimizing various temporal contrastive losses. In particular, \citet{laskin2020curl} proposes to interleave minimizing an InfoNCE objective with similarities to MoCo \cite{he2019moco} and SimCLR \cite{chen2020simclr} while learning a policy with SAC \cite{haarnoja2018soft}. Moreover, \citet{spr} proposes learning a representation similar to BYOL \cite{grill2020byol} alongside a Q-function for Deep Q-Learning \cite{minh2013atari}. We compare our representational objective against CURL and SPR in \cref{sec:exp}, and demonstrate that under linear evaluation protocol, ours outperform both CURL and SPR. Note, we refer the readers to \citet{spr} for comparisons between SPR and CPC, ST-DIM, and DRIML.

\section{Preliminaries}
\label{sec:prelim}

In this paper we consider the infinite horizon discounted MDP $\langle \mathcal{S}, \mathcal{A}, \gamma, P, r, d_0 \rangle$. where $\mathcal{S}, \mathcal{A}$ are state and action spaces which could contain a large number of states and actions or could even be infinite, $\gamma \in (0,1)$ is a discount factor, $P$ is the transition kernel, $r: \mS \times \mA \to \RR$ is the reward function, and $d_0\in\Delta(\mathcal{S})$ is the initial state distribution. 
We assume rewards are bounded by 1, i.e. $\abs{r(s,a)} \leq 1$.
Given a policy $\pi: \mathcal{S}\mapsto \Delta(\mathcal{A})$, we denote $V^{\pi}(s) = \mathbb{E}\left[  \sum_{h=0}^{\infty} \gamma^h r_h | \pi, s_0 := s \right]$ as the expected discounted total reward of policy $\pi$ starting at state $s$. We denote $V^{\pi}_{d_0} = \mathbb{E}_{s\sim d_0} V^{\pi}(s)$ as the expected discounted total reward of the policy $\pi$ starting at the initial state distribution $d_0$. We also denote average state-action distribution $d_{d_0}^{\pi}(s,a) = (1-\gamma) \sum_{h=0}^{\infty}\gamma^h d^{\pi}_h(s,a)$ where $d^{\pi}_h(s,a)$ is the probability of $\pi$ visiting the $(s,a)$ pair at time step $h$, starting from $d_0$.

In OPE, we seek to evaluate the expected discounted total reward $V^{\pi_e}$ of a target policy $\pi_e: \mathcal{S}\mapsto \Delta(\mathcal{A})$ given data drawn from an offline state-action distribution $\nu\in\Delta(\mathcal{S}\times\mathcal{A})$. The latter can, for example, be the state-action distribution under a behavior policy $\pi_b$. Namely, the offline dataset ${\Dcal} = \{s_i, a_i, r_i, s'_i\}_{i=1}^N$ consists of  $N$ i.i.d tuples generated as $(s,a) \sim \nu, r = r(s,a), s' \sim P(\cdot | s,a)$.

We define Bellman operator  $\Tcal^{\pi}$associated with $\pi$ as follows: 
\begin{align*}
   \mT^\pi f(s,a) \defeq r(s,a) + \gamma \Eb[s' \sim P(s,a), a' \sim \pi(s')]{ f(s',a') }
\end{align*}
We may drop the superscript $\pi$ when it is clear from context, in particular when $\pi=\pi_e$ is the fixed target policy.

A representation, or feature, $\phi: \mS \times \mA \to \RR^d$ is an embedding of state-action pairs into $d$-dimensional space. 
We let $\Sigma(\phi)=\Eb[\nu]{\phi(s,a)\phi(s,a)\tr},\,\wh\Sigma(\phi)=\frac1N\sum_{i=1}^N\phi(s_i,a_i)\phi(s_i,a_i)\tr$.
We consider learning a representation from  a feature hypothesis class $\Phi \subset[\Scal\times\Acal \mapsto \mathbb{R}^d]$. We assume features have bounded norm: $\nm{\phi(s,a)}_2 \leq 1, \forall s, a, \forall \phi \in \Phi$. In our experiments, $\Phi$ is convolutional neural nets with $d$ outputs. 

\paragraph{Notation} We denote $B_{W} := \{ x\in \mathbb{R}^d : \|x\|_2 \leq W \}$ as the Euclidean Ball in $\mathbb{R}^d$ with radius $W$. Given a distribution $\nu \in \Delta(\Scal\times\Acal)$ and a function $f:\Scal\times\Acal\mapsto \mathbb{R}$, we denote $L_2(\nu)$ norm as $\| f\|_\nu^2 = \mathbb{E}_{s,a\sim \nu} f^2(s,a)$. Given a positive definite matrix $\Sigma$, let $\|x\|_\Sigma = \sqrt{x^T\Sigma x}$ and let $\lambda_{\min}( \Sigma )$ denote the minimum eigenvalue.
When $\nu \ll \mu$ we let $\frac{\diff \nu}{\diff \mu}$ denote the Radon-Nikodym derivative.
We use $\circ$ to denote composition, so $s',a' \sim P(s,a) \circ \pi$ is short-form for $s' \sim P(s, a), a' \sim \pi(s')$.
For any function $f(s,a)$ and a policy $\pi$, we denote $f(s,\pi) = \Eb[a \sim \pi(s)]{f(s,a)}$.

\section{Linear Bellman Completeness}\label{sec:lbc_rep}

Before we introduce our representation learning approach, in this section, we first consider OPE with linear function approximation with a given representation $\phi$.
Lessons from supervised learning or linear bandit may suggest that OPE should be possible with polynomially many offline samples, as long as (1) ground truth target $Q^{\pi_e}$ is linear in $\phi$ (i.e. $\exists w\in\mathbb{R}^d$, such that for all $s,a$, $Q^{\pi_e}(s,a) = w^{\top} \phi(s,a)$), and (2) the offline data provides sufficient coverage over $\pi_e$ (i.e., $\lambda_{\min}(\Sigma(\phi))>0$).
Unfortunately, under these two assumptions, there are lower bounds indicating that for \emph{any} OPE algorithm, there exists an MDP where one will need at least exponentially (exponential in horizon or $d$) many offline samples to provide an accurate estimate of $V^{\pi_e}$ \citep{wang2021exponential,foster2021offline,amortila2020variant}.
This lower bound indicates that one needs additional structural conditions on the representation.  \looseness=-1

The additional condition the prior work has consider is Bellman completeness (BC). Since we seek to learn a representation rather than assume theoretical conditions on a given one, we will focus on an \emph{approximate} version of BC.
 \begin{definition}[Approximate Linear BC] A representation $\phi$ is $\eps_\nu$-approximately linear Bellman complete if,
\begin{equation*} 
{\displaystyle \max_{w_1 \in B_W} \min_{w_2 \in B_W}  \left\|   w_2\tr \phi   -\Tcal^{\pi_e} ( w_1\tr \phi )\right\|_{\nu} \leq \eps_\nu}.
\end{equation*} 
Note the dependence on $\nu$, $\pi_e$, and $W$.
\label{def:lbe}
\end{definition}
Intuitively the above condition is saying that for any linear function $w_1\tr \phi(s,a)$, $\Tcal^{\pi} (w_1\tr \phi(s,a))$ itself can be approximated by another linear function under the distribution $\nu$.

\begin{remark}
Low-rank MDPs \citep{jiang2017contextual,agarwal2020flambe} are subsumed in the exact linear BC model, i.e., $\eps_\nu = 0$ under any distribution $\nu$.
\end{remark}

Note that the Bellman completeness condition is more subtle than the common realizability condition: for any fixed $\phi$, increasing its expressiveness (e.g., add more features) does not imply a monotonic decrease in $\eps_\nu$. Thus common tricks such as lifting features to higher order polynomial kernel space or more general reproducing kernel Hilbert space does not imply the linear Bellman complete condition, nor does it improve the coverage condition.
We next show approximate Linear BC together with coverage imply sample-efficient OPE via Least Square Policy Evaluation (LSPE) (\cref{alg:lspe}).
We present our result using the relative condition number, defined for any initial state distribution $p_0$ as 
\begin{align}
    \label{eq:relative_cond}
    \kappa(p_0) := \sup_{x \in \RR^d} \frac{x\tr \EE_{s,a\sim d_{p_0}^{\pi_e}} \phi(s,a)\phi(s,a){\tr} x }{x\tr  \Sigma(\phi) x}.
\end{align}

\begin{algorithm}[t]
	\caption{Least Squares Policy Evaluation (LSPE)}
	\label{alg:lspe}
	\begin{algorithmic}[1]
	    \STATE \textbf{Input: } Target policy $\pi_e$, features $\phi$, dataset $\mD$
	    \STATE Initialize $\wh\theta_0 = \textbf{0} \in B_W$.
	    \FOR{$k = 1, 2, ..., K$}
	    	\STATE Set $\wh f_{k-1}(s, a) = \wh\theta_{k-1}\tr \phi(s,a)$, \\\phantom{Set }$\wh V_{k-1}(s)=\EE_{a \sim \pi_e(s)}[\wh f_{k-1}(s, a)]$
	        \STATE Perform linear regression: 
	        \vspace{-5pt}\begin{align*}
	        	  \wh\theta_k \in  \argmin_{\theta \in B_W} 
	        	  &\frac{1}{N}\sum_{i=1}^N \big( \theta\tr \phi(s_i,a_i) - r_i  - \gamma \wh V_{k-1}(s_i') \big)^2
	        \end{align*} \label{line:least_square_lspe}
	        \vspace{-15pt}
		\ENDFOR
		\STATE Return $\wh f_K$.
	\end{algorithmic}
\end{algorithm}

\begin{restatable}[Sample Complexity of LSPE]{theorem}{fqesamplecomplexity} \label{thm:fqe}
Assume feature $\phi$ satisfies approximate Linear BC with parameter $\eps_\nu$.
For any $\delta \in (0, 0.1)$, w.p. at least $1-\delta$, we have for any initial state distribution $p_0$
\begin{align*}
&\abs{V_{p_0}^{\pi_e}- \Eb[s\sim p_0]{\wh f_K(s,\pi_e)}}
\leq
\frac{\gamma^{K/2}}{1-\gamma} + \frac{4 \sqrt{\nm{\frac{\diff d_{p_0}^{\pi_e}}{\diff \nu}}_\infty} }{(1-\gamma)^2} \eps_\nu 
\\&+ \frac{480\sqrt{\kappa(p_0)}(1+W)d\log(N) \sqrt{\log(10/\delta)}}{(1-\gamma)^2 \sqrt{N}},
\end{align*}
where $\wh f_K$ is the output of \cref{alg:lspe}.
\end{restatable}
Please see \cref{sec:proof-fqe} for proof.
The above result holds simultaneously over all initial state distributions $p_0$ covered by the data distribution $\nu$.
If $\nu$ has full coverage, i.e. if $\Sigma \succeq \beta I$, as is commonly assumed in the literature, one can show that $\kappa(p_0) \leq \beta^{-1}$ for any initial state distribution $p_0$.
Also note that the concentrability coefficient $\nm{\frac{\diff d_{p_0}^{\pi_e}}{\diff \nu}}_\infty$ shows up as $\Tcal^{\pi_e} \wh f_{k-1}$ can be \emph{nonlinear}.  In the exact Linear BC case, where $\eps_\nu=0$ (i.e., there is a linear function that perfectly approximates $\Tcal^{\pi_e}\wh f_{k-1}$ under $\nu$), the term involving the concentrability coefficient will be 0 and we can even avoid its finiteness.

\section{Representation Learning}\label{sec:learning-features}

The previous section indicates sufficient conditions on the representation for efficient and accurate OPE with linear function approximation. However, a representation that is Bellman complete and also provides coverage is not available a priori, especially for high-dimensional settings (e.g., image-based control tasks as we consider in our experiments). Existing theoretically work  often assume such representation is given. 
We propose to learn a representation $\phi$ that is approximately Bellman complete and also provides good coverage, via rich function approximation (e.g., a deep neural network). More formally, we want to learn a representation $\phi$ such that (1) it ensures approximate Bellman complete, i.e., $\varepsilon_\nu$ is small, and (2) has a good coverage, i.e., $\lambda_{\min}(\Sigma(\phi))$ is as large as possible, which are the two key properties to guarantee accurate and efficient OPE indicated by Theorem~\ref{thm:fqe}.
To formulate the representation learning question, our key assumption is that our representation hypothesis class $\Phi$ is rich enough such that it contains at least one representation $\phi^\star$ that is linear Bellman complete (i.e., $\eps_\nu = 0$) and has a good coverage.

\begin{assumption}[Representational power of $\Phi$] \label{assum:phi_condition} 
There exists $\phi^\star\in \Phi$, such that (1) $\phi^\star$ achieves exact Linear BC (definition~\ref{def:lbe} with $\varepsilon_{\nu} = 0$), and (2) $\phi^\star$ induces coverage, i.e., $\lambda_{\min}(\Sigma(\phi^\star)) \geq \beta >0$. 
\vspace{-5pt}
\end{assumption}
Our goal is to learn such a representation from the hypothesis class $\Phi$. Note that unlike prior RL representation learning works \citep{hao2021sparse, zhang2021provably}, here we only assume that there exists a $\phi^\star$ has linear BC and induces good coverage, other candidate  $\phi\in \Phi$ could have terrible coverage and does not necessarily have linear BC.

Before proceeding to the learning algorithm, we first present an \emph{equivalent condition} for linear BC, which does not rely on the complicated min-max expression of \cref{def:lbe}.
\begin{restatable}{proposition}{equivalentlbc}\label{prop:equivalent_lbc}
Consider a feature $\phi$ with full rank covariance $\Sigma(\phi)$. Given any $W > 0$, the feature $\phi$ being linear BC (under $B_W$) implies that there exist $(\rho,M)\in B_W\times \mathbb{R}^{d\times d}$ with $\|M\|_2  \leq \sqrt{ 1 - \frac{ \|\rho\|^2  }{ W^2 }}$, and 
\begin{align}
\label{eq:linear_M_condition}
\mathbb{E}_{\nu} \left\| \begin{bmatrix} M \\ \rho\tr\end{bmatrix} \phi(s,a) -  \begin{bmatrix}  \gamma \EE_{s'\sim P(s,a)}\phi(s', \pi_e) \\ r(s,a) \end{bmatrix}\right\|^2_2  = 0.
\end{align}
On the other hand, if there exists $(\rho, M)\in B_W\times\RR^{d\times d}$ with $ \|M\|_2 < 1$ such that the above equality holds, then $\phi$ must satisfy exact linear BC with $W \geq\frac{\|\rho\|_2}{ 1- \|M\|_2 }$.
\end{restatable}
Please see \cref{sec:equivalent_lbc} for proof.
The above shows that linear BC is equivalent to a simple linear relationship between the feature and the expected next step's feature and reward. 
This motivates our feature learning algorithm: if we are capable of learning a representation $\phi$ such that the transition from the current feature $\phi(s,a)$ to the expected next feature $\mathbb{E}_{s'\sim P(s,a)}[\phi(s',\pi_e)]$ and reward $r(s,a)$ is linear, then we've found a feature $\phi$ that is linear BC.

\subsection{Algorithm}
\label{subsec:alg}

To learn the representation that achieves linear BC, we use \cref{prop:equivalent_lbc} to design a bilevel optimization program.  We start with deterministic transition as a warm up.

\paragraph{Deterministic transition} Due to determinism in the transition, we do not have an expectation with respect to $s'$ anymore. So, we design the bilevel optimization as follows:
\begin{align}
&\min_{\phi\in \Phi } \Big[ \min_{(\rho, M) \in \Theta} \mathbb{E}_{\mD} \left\| \begin{bmatrix} M \\ \rho\tr \end{bmatrix} \phi(s,a) -  \begin{bmatrix} \gamma \phi(s', \pi_e) \\ r(s,a) \end{bmatrix}\right\|^2_2\label{eq:bilevel_deter} 
\\&\Theta = \braces{ (\rewardVec, M) \in B_{\|\rho^\star\|} \times \RR^{d\times d}: , \|M\|_2\leq\|M^\star\|_2 }, \nonumber
\end{align}  
where $\rho^\star, M^\star$ are the optimal $\rho,M$ for the linear BC $\phi\opt$ (in \cref{assum:phi_condition}).
Namely, we aim to search for a representation $\phi\in \Phi$, such that the relationship between $\phi(s,a)$ and the combination of the next time step's feature $\phi(s',\pi_e)$ and the reward, is linear. 
The spectral norm constraint in $\Theta$ is justified by \cref{prop:equivalent_lbc}. Solving the above bilevel moment condition finds a representation that achieves approximate linear BC.

However, there is no guarantee that such representation can provide a good coverage over $\pi_e$'s traces. 
We introduce regularizations for $\phi$ using the ideas from optimal designs, particularly the $E$-optimal design. Define the minimum eigenvalue regularization
\begin{align*}
R_E(\phi) := \lambda_{\min}\left( \Eb[\mD]{\phi(s,a) \phi(s,a)\tr} \right), 
\end{align*}
as the smallest eigenvalue of the empirical feature covariance matrix under the representation $\phi$.
Maximizing this quantity ensures that our feature provides good coverage, i.e. $\lambda_{min}(\Sigma(\phi))$ is as large as possible.
 
Thus, to learn a representation that is approximately linear Bellman complete and also provides sufficient coverage, we formulate the following constrained bilevel optimization:
\begin{align*}
&\min_{\phi\in \Phi } \Big[ \min_{(\rho, M) \in \Theta} \mathbb{E}_{\mD} \left\| \begin{bmatrix} M \\ \rho\tr \end{bmatrix} \phi(s,a) -  \begin{bmatrix} \gamma \phi(s', \pi_e) \\ r(s,a) \end{bmatrix}\right\|^2_2\\
& s.t., R_{E}(\phi )\geq \beta/2.
\end{align*}\vspace{-20pt}

To extend this to stochastic transitions, there is an additional double sampling issue, which we discuss and address now.
\paragraph{Stochastic transition}
Ideally, we would solve the following bilevel optimization problem,
{\small
\begin{align}\label{eq:bilevel_opt_stoch_no_double_sampling}
\min_{\phi\in \Phi } \Big[ \min_{(\rho, M) \in \Theta}  \mathbb{E}_{\mD} \left\| \begin{bmatrix} M \\ \rho\tr \end{bmatrix} \phi(s,a) -  \begin{bmatrix}  \gamma \EE_{s'\sim P(s,a)}[ \phi(s', \pi_e)] \\ r(s,a) \end{bmatrix}\right\|^2_2\Big].
\end{align}}
However, we cannot directly optimize the above objective since we do not have access to $P(s,a)$ to compute the expected next step feature $\EE_{s'\sim P(s,a)}\phi(s', \pi_e)$.
Also note that the expectation $\EE_{s'}$ is inside the square which means that we cannot even get an unbiased estimate of the gradient of $(\phi, M)$ with one sample $s'\sim P(s,a)$. This phenomenon is related to the \emph{double sampling} issue on offline policy evaluation literature which forbids one to directly optimize Bellman residuals. Algorithms such as TD are designed to overcome the double sampling issue.
Here, we use a different technique to tackle this issue \citep{chen2019information}. We introduce a function class $\mathcal{G}\subset \Scal\times\Acal\mapsto \mathbb{R}^d$ which is rich enough to contain the expected next feature.
\begin{assumption}\label{asm:g-realizability}
For any $\phi \in \Phi$, we have that the mapping $(s, a) \mapsto \gamma \Eb[s'\sim P(s,a)]{\phi(s',\pi_e)}$ is in $\mG$.
\end{assumption}
We form the optimization problem as follows:
\begin{align}
& \min_{\phi\in \Phi }  \Big[ \min_{(\rho, M) \in \Theta} \mathbb{E}_{\mD} \left\| \begin{bmatrix} M \\ \rho\tr \end{bmatrix} \phi(s,a) -  \begin{bmatrix} \gamma \phi(s', \pi_e) \\ r(s,a) \end{bmatrix}\right\|^2_2 \nonumber
\\& \quad  -\min_{g\in\mathcal{G}} \EE_{\mD}\left\| g(s,a) - \gamma \phi(s', \pi_e) \right\|_2^2   \Big ]. \label{eq:bilevel_opt_stoch}
\end{align}
Note that under \cref{asm:g-realizability}, the min over $\mG$ will approximate $\gamma^2 \EE_{D}\| \EE_{s'\sim P(s,a)} \phi(s',\pi_e) - \phi(s', \pi_e)  \|_2^2$, i.e., the average variance induced by the stochastic transition.
Thus, for a fixed $\phi$ and $M$, we can see that the following 
\begin{align*}
&\mathbb{E}_{\mD} \left\|  M\phi(s,a) -  \gamma \phi(s', \pi_e) \right\|^2_2   \\
 & - \gamma^2 \EE_{D}\| \EE_{s'\sim P(s,a)} \phi(s',\pi_e) - \phi(s', \pi_e)  \|_2^2
\end{align*} is indeed an unbiased estimate of:
\begin{align*}\vspace{-10pt}
 \mathbb{E}_{s,a\sim \nu} \left\| M\phi(s,a) - \gamma \EE_{s'\sim P(s,a)}\phi(s', \pi_e) \right\|^2_2
\end{align*} which matches to the ideal objective in Eq.~\ref{eq:bilevel_opt_stoch_no_double_sampling}.
Thus solving for $\phi$ based Eq.~\ref{eq:bilevel_opt_stoch} allows us to optimize Eq.~\ref{eq:bilevel_opt_stoch_no_double_sampling}, which allows us to learn an approximate linear Bellman complete representation.
Similarly, we incorporate the $E$-optimal design here by adding a constraint that forcing the smallest eigenvalue of the empirical feature covariance matrix, i.e., $R_E(\phi)$, to be lower bounded.

Once we learn a representation that is approximately linear BC, and also induces sufficient coverage, we simply call the LSPE  to estimate $V^{\pi_e}$. The whole procedure is summarized in Algorithm~\ref{alg:feature_learning}. Note in Alg~\ref{alg:feature_learning} we put constraints to the objective function using Lagrangian multiplier. \looseness=-1

There are other design choices that also encourage coverage. One particular choice we study empirically is motivated by the idea of D-optimal design. Here we aim to find a representation that maximizes the following log-determinant %
\begin{align*}
R_{D}( \phi ) := \ln\text{det}\left( \Eb[\mD]{\phi(s,a) \phi(s,a)\tr} \right).
\end{align*}
When $\Dcal$ is large, then the regularization $R_D(\phi)$ approximates $\sum_i \ln \left( \sigma_i \left( \Eb[\nu]{ \phi(s,a) \phi(s,a)\tr} \right)\right)$. Maximizing $R_D(\phi)$ then intuitively maximizes coverage over all directions.  D-optimal design is widely used for bandits \citep{lattimore2020bandit} and RL \cite{wang2021exponential,rltheorybookAJKS} to design exploration distributions with global coverage. Note that, in contrast to these contexts where the feature is fixed and the distribution is optimized, we optimize the feature, given the data distribution $\nu$.
\looseness=-1

\begin{algorithm}[t]
	\caption{OPE with \textbf{B}ellman \textbf{C}omplete and exploratory \textbf{R}epresentation \textbf{L}earning (\alg)}
	\label{alg:feature_learning}
	\begin{algorithmic}[1]
	    \STATE \textbf{Input: } Representation class $\Phi$, dataset $\mD$ of size $2N$, design regularization $R$, function class $\mathcal{G}$, policy $\pi_e$.
	    \STATE Randomly split $\mD$ into two sets $\Dcal_1, \Dcal_2$ of size $N$.
	    \STATE If the system is stochastic, learn representation $\wh \phi$ as,
	    \begin{align*}
	    	&\argmin_{\phi\in \Phi } \Big[ \min_{(\rho,M)\in\Theta}  \mathbb{E}_{\mD_1} \left\| \begin{bmatrix} M \\ \rho\tr \end{bmatrix} \phi(s,a) -  \begin{bmatrix} \gamma \phi(s', \pi_e) \\ r(s,a) \end{bmatrix}\right\|^2_2 \\
		& \quad  - \min_{g\in\mathcal{G}} \mathbb{E}_{\mD_1} \left\| g(s,a) - \gamma \phi(s',\pi_e) \right\|_2^2 \Big ] - \lambda R(\phi)
	    \end{align*}
	    \STATE Otherwise, for deterministic system, learn $\wh \phi$ as,
	    \begin{align*}
	    &\argmin_{\phi\in \Phi } \Big[ \min_{(\rho,M)\in\Theta}  \mathbb{E}_{\mD_1} \left\| \begin{bmatrix} M \\ \rho\tr \end{bmatrix} \phi(s,a) -  \begin{bmatrix} \gamma \phi(s', \pi_e) \\ r(s,a) \end{bmatrix}\right\|^2_2 \\
		& \quad  - \lambda R(\phi)
	    \end{align*}
	    \STATE Return $\wh V := \text{LSPE}(\pi_e, \wh\phi, \mD_2)$.
	\end{algorithmic}
\end{algorithm}

\subsection{Sample Complexity Analysis}
We now prove a finite sample utility guarantee for the empirical constrained bilevel optimization problem,
\begin{align}
&\wh\phi \in \argmin_{\phi\in \Phi }  \Big[ \min_{(\rewardVec, M) \in \Theta} \mathbb{E}_{\mD} \left\|  \begin{bmatrix} M \\ \rewardVec^{\tr} \end{bmatrix} \phi(s,a) -  \begin{bmatrix} \gamma\phi(s', \pi_e) \\ r(s,a) \end{bmatrix}\right\|^2_2 \nonumber \\
&   \quad  -\min_{g\in\mathcal{G}} \EE_{\mD}\left\| g(s,a) - \gamma\phi(s', \pi_e) \right\|_2^2   \Big ]. \label{eq:phihat-def}
\\ & s.t., R_E(\phi) \geq \beta/2.  \nonumber
\end{align}\vspace{-10pt}

For simplicity, we state our results for discrete function class $\Phi$ and $\mathcal{G}$. Note that our sample complexity only scales with respect $\ln(|\Phi|)$ and $\ln(|\mathcal{G}|)$, which are the standard complexity measures for discrete function classes. We extend our analysis to infinite function classes under metric entropy assumptions \citep{wainwright2019high,vandervaart1996weak} in the Appendix; see \cref{asm:entropy}.

The following theorem shows that \cref{alg:feature_learning} learns a representation $\wh\phi$ that is $\mO(N^{-1/2})$ approximately Linear BC and has coverage.

\begin{restatable}{theorem}{learningphi}\label{thm:learning-phi}
Assume \cref{assum:phi_condition} (and \cref{asm:g-realizability} if the system is stochastic).
Let $C_2 \coloneqq \frac{96\log^{1/2}(|\Phi|) + 4\sqrt{d} + 4\log^{1/2}(1/\delta)}{\beta/4}$. 
If $N \geq C_2^2$, then for any $\delta \in (0, 1)$, w.p. at least $1-\delta$, we have
\begin{enumerate}[leftmargin=*]
	\item $\wh\phi$ satisfies $\wh\eps$-approximate Linear BC with
    \begin{align*}
        \wh\eps 
        &\leq \frac{13d(1+W)^2\log^{1/2}(4W|\Phi|N/\delta)}{\sqrt{N}} 
        \\&+ \frac{7\gamma(1+W)\log^{1/2}(2|\mG|/\delta)}{\sqrt{N}},
    \end{align*}
	\item $\lambda_{min}(\Sigma(\wh\phi)) \geq \beta/4$.
\end{enumerate}
If transitions are deterministic, treat $\log(|\mG|) = 0$. %
\end{restatable}
\begin{proof}[Proof Sketch]
First, we use Weyl's Perturbation Theorem and chaining to show that the eigenvalues of $\Sigma(\phi)$ are close to $\wh\Sigma(\phi)$, uniformly over $\phi$. This implies that (a) $\lambda_{min}(\wh\Sigma(\phi\opt)) \geq \beta/2$, and hence is feasible in the empirical bilevel optimization \cref{eq:phihat-def}, and (b) $\lambda_{min}(\Sigma(\wh\phi)) \geq \beta/4$.
Since $\phi\opt$ is feasible, we apply uniform concentration arguments to argue that $\wh\phi$ has low population loss (\cref{eq:bilevel_opt_stoch}), which implies approximate Linear BC.
\end{proof}
The error term in $\wh\eps$ is comprised of the statistical errors of fitting $\Phi$ and of fitting $\mG$ for the double sampling correction.
In the contextual bandit setting, i.e. $\gamma=0$, there is no transition, so the second term  becomes 0.
Chaining together with Theorem~\ref{thm:fqe} gives the following end-to-end $\wt \mO(N^{-1/2})$ evaluation error guarantee for LSPE with the learned features:
\begin{restatable}{theorem}{endtoend}\label{thm:end-to-end}
Under \cref{assum:phi_condition} (and \cref{asm:g-realizability} if $P(s,a)$ is stochastic).
Let $C_2, \wh\eps$ be as defined in \cref{thm:learning-phi}.
If $N \geq C_2^2$, then we have for any $\delta \in (0, 1/2)$, w.p. at least $1-2\delta$, for all distributions $p_0$, 
\begin{align*}
&\abs{V_{p_0}^{\pi_e}- \Eb[s\sim p_0]{\wh f_K(s,\pi_e)}}
\leq \frac{\gamma^{K/2}}{1-\gamma} +\frac{4\sqrt{\nm{\frac{\diff d_{p_0}^{\pi_e}}{\diff \nu}}_\infty}}{(1-\gamma)^2} \wh\eps
\\&+ \frac{960\beta^{-1/2}(1+W)d\log(N) \sqrt{\log(10/\delta)}}{(1-\gamma)^2 \sqrt{N}}.
\end{align*}
\end{restatable}

\paragraph{Comparison to FQE}  What if one ignores the representation learning and just runs the Fitted Q-Evaluation (FQE) which directly performs least square fitting with the nonlinear function class $\mathcal{F} := \{w^{\top} \phi(s,a) : \phi\in \Phi, \|w\|_2 \leq W\}$? 
As $N\to\infty$, FQE will suffer the following worst-case Bellman error (also called inherent Bellman error):
\begin{align*}
\varepsilon_{\mathrm{ibe}} := \max_{f \in\mathcal{F}} \min_{g \in \mathcal{F}} \left\| g - \Tcal^{\pi_e} f \right\|^2_{\nu}.
\end{align*}
Note that our assumption that there exists a linear BC representation $\phi^\star$ does not imply that the worst-case Bellman error $\varepsilon_{ibe}$ is small.  In contrast, when $N\to\infty$, our approach will accurately estimate $V_{p_0}^{\pi_e}$.

\begin{figure}[b]
    \centering
    \includegraphics[width=\linewidth]{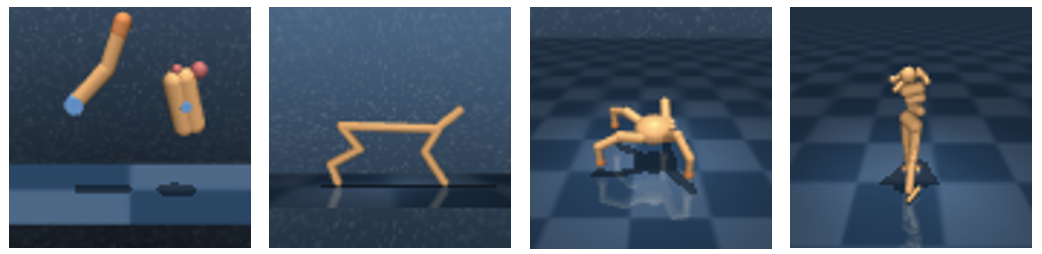}
    \vspace{-6.0mm}
    \caption{Representative frames from DeepMind Control Suite tasks. From left to right, \texttt{Finger Turn Hard}, \texttt{Cheetah Run}, \texttt{Quadruped Walk}, and \texttt{Humanoid Stand}.}
    \label{fig:envs}
    \vspace{-10pt}
\end{figure}

\begin{figure*}[ht]
    \centering
    \includegraphics[width=\textwidth]{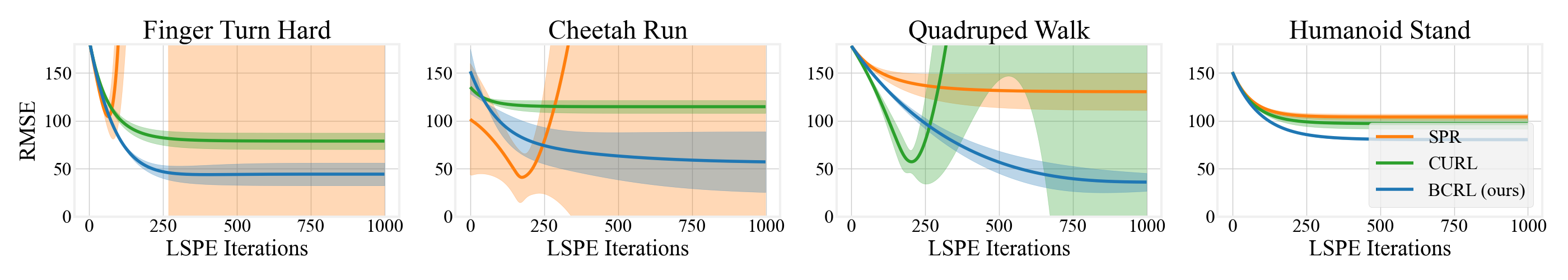}
    \vspace{-5mm}
    \caption{OPE curves across five seeds using representations trained with~\alg, SPR, and CURL on the offline datasets (\cref{tab:perf}).}
    \vspace{-10pt}
    \label{fig:eval}
\end{figure*}

\section{A Practical Implementation}
\begin{algorithm}[t]
	\caption{Practical Instantiation of~\alg}
	\label{alg:prac_feature_learning}
	\begin{algorithmic}[1]
	    \STATE \textbf{Input: } Offline dataset $\mD = \{s, a, s'\}$, target policy $\pi_e$
	    \STATE Initialize parameters for $\phi$, $M$, $\rewardVec$, and $\overline{\phi}$
	    \FOR {t = 0, 1, \ldots, T}
	        \STATE $M_{t+1} \leftarrow M_t - \eta\nabla_{M}J(\phi_t, M_t, \rewardVec_t, \overline{\phi}_t)$
	        \STATE $\rewardVec_{t+1} \leftarrow \rewardVec_t - \eta\nabla_{\rewardVec}J(\phi_t, M_t, \rewardVec_t, \overline{\phi}_t)$
	        \STATE $\phi_{t+1} \leftarrow \phi_t - \eta\nabla_{\phi}J(\phi_t, M_{t+1}, \rewardVec_{t+1}, \overline{\phi}_t)$
	        \STATE $\overline{\phi}_{t+1} \leftarrow \tau\phi_{t+1} + (1-\tau)\overline{\phi}_t$
	    \ENDFOR
	    \STATE Linear evaluation: $\wh V = \text{LSPE}(\pi_e, \phi_T, \mD)$.
	\end{algorithmic}
\end{algorithm}

In this section we instantiate a practical implementation to learn our representation using deep neural networks for our representation function class $\Phi$. Based on \cref{eq:bilevel_deter}, we first formalize our bilevel optimization objective:
\begin{align*}
J(\phi, M, \rewardVec, \overline{\phi}) &=  \mathbb{E}_{\mD} \left\|  \begin{bmatrix} M \\ \rewardVec^{\tr} \end{bmatrix}\phi(s,a) -  \begin{bmatrix} \gamma \overline{\phi}(s', \pi_e) \\ r(s,a) \end{bmatrix}\right\|^2_2 \\
& - \lambda \log\det\mathbb{E}_{\mD}\left[\phi(s,a)\phi(s,a)^\top\right].
\end{align*}
In our implementation, we replace the hard constraint presented in \cref{subsec:alg} with a Lagrange multiplier, i.e. we use the optimal design constraint as a regularization term when learning $\phi$. Specifically, we maximize the $\log\det$ of the covariance matrix induced by the feature, which maximizes all eigenvalues since $\log\det$ is the sum of the log eigenvalues. Our experiment results in \cref{sec:ablations} demonstrate that it indeed improves the condition number of $\Sigma(\phi)$. In some of our experiments, we use a target network in our implementation. Namely, the updates make use of a target network $\overline{\phi}$ where the weights can be an exponential moving average of the representation network's weights. This idea of using a slow moving target network has been shown to stabilize training in both the RL \cite{minh2013atari} and the representation learning literature \cite{grill2020byol}. 

As summarized in \cref{alg:prac_feature_learning}, given our offline behavior dataset $\mD$ and target policy $\pi_e$, we iteratively update $M$ and $\phi$ and then use the resulting learned representation to perform OPE. Please see Appendix \ref{app:implementation} for implementation and hyperparameter details. As we will show in our experiments, our update procedures for $\phi$ significantly minimizes the Bellman Completeness loss and also improve the condition number of $\Sigma(\phi)$, which are the two key quantities to ensure good performance of LSPE with linear regression as shown in Theorem~\ref{thm:fqe}. \looseness=-1

\section{Experiments}
\label{sec:exp}

Our goal is to answer the following questions. (1) How do our representations perform on downstream LSPE compared to other popular unsupervised representation learning techniques? 
(2) How important are both the linear bellman completeness and optimal design components for learning representations? 
(3) How competitive is \alg~with other OPE methods, especially for evaluating beyond the initial state distribution?

Following the standards in evaluating representations in supervised learning, we used a \textbf{linear evaluation protocol}, i.e., linear regression in LSPE on top of a given representation.
This allows us to focus on evaluating the quality of the representation.
We compared our method to prior techniques on a range of challenging, image-based continuous control tasks from the DeepMind Control Suite benchmark \cite{tassa2018deepmind}: \texttt{Finger Turn Hard}, \texttt{Cheetah Run}, \texttt{Quadruped Walk}, and \texttt{Humanoid Stand}. Frames from the tasks are shown in \cref{fig:envs}. To investigate our learned representation, we benchmark our representation against two state-of-the-art self-supervised representation learning objectives adopted for RL: (1) CURL uses the InfoNCE objective to contrastively learn state-representations; and (2) SPR adopts the BYOL contrastive learning framework to learn representations with latent dynamics.
Note, we modified CURL for OPE by optimizing the contrastive loss between state-action pairs rather than just states. For SPR, we did not include the Q prediction head and used their state representation plus latent dynamics as the state-action representation for downstream linear evaluation. We used the same architecture for the respective representations across all evaluated algorithms. To compare against other OPE methods, we additionally compared against Fitted Q-Evaluation \citep{munos2008finite,kostrikov2020fqe} (FQE), weighted doubly robust policy evaluation \citep{jiang2016doubly,thomas2016doublyrobust} (DR), Dreamer-v2 \citep{hafner2020mastering} model based evaluation (MB), and DICE \citep{yang2020off}. We modified implementations from the benchmark library released by \citet{fu2021deepope} for FQE and DR and used the authors' released codebases for Dreamer-v2 and BestDICE.\footnote{{\scriptsize \url{https://github.com/google-research/dice_rl}}.}

Our target policies were trained using the authors' implementation of DRQ-v2 \cite{yarats2021mastering}, a state-of-the-art, off-policy actor critic algorithm for vision-based continuous control. With high-quality target policies, we collected 200 rollouts and did a linear evaluation protocol to predict the discounted return. For our offline behavior datasets, we collected 100K samples from a trained policy with mean performance roughly a quarter of that of the target policy (\cref{tab:perf}). All results are aggregated over five random seeds. See Appendix \ref{app:implementation} for details on hyperparameters, environments, and dataset composition.

\subsection{OPE via LSPE with Learned Representations}

\cref{fig:eval} compares the OPE performance of \alg~against SPR and CURL. Representations learned by \alg~outperform those learned by SPR and CURL. On some tasks, SPR and CURL both exhibited an exponential error amplification with respect to the number of iterations of LSPE, while \alg~did not suffer from any blowup.

\subsection{OPE Performance}
\cref{fig:ope_result} compares the OPE performance of \alg~against multiple benchmarks from the OPE literature. \alg~is competitive with FQE and evaluates better than other benchmarks across the tasks that we tested on.

We also evaluated how well estimated values from \alg~rank policies. Following \cite{fu2021deepope}, we use
 the spearman ranking correlation metric to compute the correlation between ordinal rankings according to the OPE estimates and the ground truth values. For ranking, we evaluated three additional target policies with mean performances roughly 75\%, 50\%, and 10\% of the target policy. \cref{fig:ope_result} presents the mean correlation of each evaluation algorithm across all tasks. \alg~is competitive with FQE and consistently better than other benchmarks at ranking policies. 

\subsection{Further Investigation of Different Settings}

\begin{figure}[b!]
    \centering
    \includegraphics[width=\linewidth]{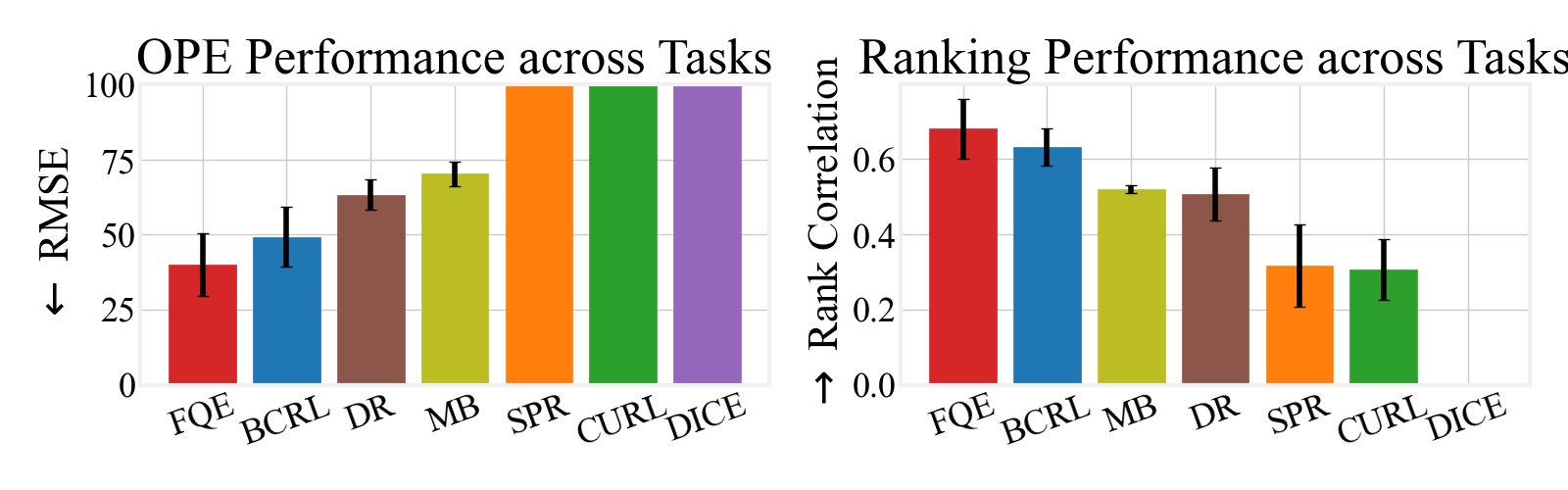}
    \caption{(Left) Root mean squared evaluation error across all tasks. (Right) Mean spearman ranking correlation across all tasks.}
    \label{fig:ope_result}
    \vspace{2\baselineskip}
    \centering
    \includegraphics[width=\linewidth]{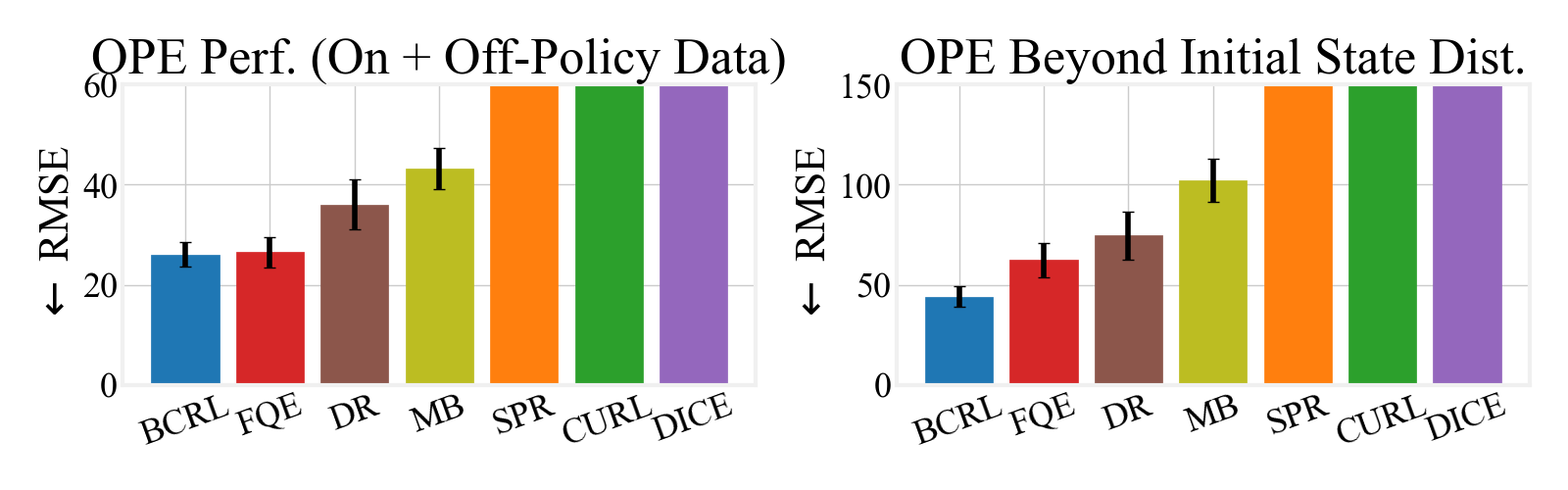}
    \caption{(Left) Evaluation on a mixture dataset with on-policy and off-policy data. Note the addition of target policy data bounds the relative condition number (Eq. \ref{eq:relative_cond}).
    (Right) Evaluation beyond the initial state distribution (given just the offline data).}
    \label{fig:ope_setting}
\end{figure}

In this section, we consider two additional settings: (1) the offline dataset contains some on-policy data, which ensures that the offline data provides coverage over the evaluation policy's state action distribution; (2) we evaluate all estimators beyond the original initial state distribution $d_0$. The first setting ensures that baselines and our algorithm all satisfy the coverage condition, thus demonstrating the unique benefit of learning a Bellman complete representation. The second setting evaluates the robustness of our algorithm, i.e., the ability to estimate beyond the original $d_0$.

\begin{figure*}[h]
    \centering
    \includegraphics[width=\textwidth]{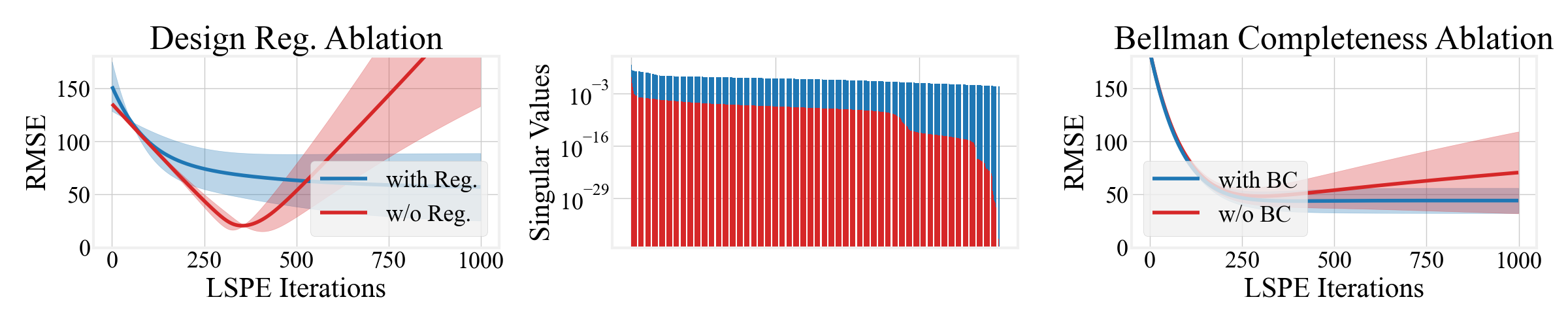}
    \vspace{-5mm}
    \caption{(Left) \alg's OPE curves for \texttt{Cheetah Run} with (\textcolor{blue}{blue}) and without (\textcolor{red}{red}) the $D$-optimal design based regularization. 
    (Center) Bar plot of the singular values of the feature covariance matrices for the left plot. 
    (Right) OPE curves for \texttt{Finger Turn Hard} with (\textcolor{blue}{blue}) and without (\textcolor{red}{red}) optimizing for linear BC. \looseness=-1}
    \label{fig:ablation}
\end{figure*}

\textbf{Evaluation on On-Policy + Off-Policy Data: }
To further investigate the importance of learning linear BC features, we experiment with learning representations from an offline dataset that also contains state-action pairs from the target policy. More specifically, we train our representations on a dataset containing 100K behavior policy and 100K target policy samples. Note, only the experiment in this paragraph uses this mixture dataset. With the addition of on-policy data from the target policy, we can focus on just the role of linear BC for OPE performance because the density ratio $\frac{\diff d^{\pi_e}_{p_o}}{\diff\nu}$ and the relative condition number (Eq.~\ref{eq:relative_cond}) is at most $2$, i.e. we omit the design regularization and focused on minimizing the Bellman completeness loss.
\cref{fig:ope_setting} (Left) shows that \alg~outperforms baselines in this setting, even matching FQE performance across tasks. Our experiments corroborate our analysis that explicitly enforcing our learned representations to be linear BC improves OPE. Note that the fact that LSPE with SPR and CURL features still blows up under this mixture data means that the other representations' failures are not just due to coverage. 

\textbf{Evaluation Beyond Initial State Distribution: }
To further investigate the benefits of optimizing the $D$-optimal design to improve coverage, we investigate doing OPE \emph{beyond} the initial state distribution $d_0$, which is supported by \cref{thm:end-to-end}. Note that if our representation is exactly Bellman complete and also the corresponding feature covariance matrix is well-conditioned, we should be able to evaluate well on any states.  
Specifically, we experiment on evaluating at all timesteps in a target policy rollout, not just at the initial state distribution. \cref{fig:ope_setting} (Right) shows that \alg~is able to more robustly evaluate out-of-distribution than all other benchmarks.

\subsection{Ablation Studies}\label{sec:ablations}

\textbf{Impact of Optimal Design Regularization: }
To investigate the impact of maximizing the $D$-optimal design, we ablate the design regularization term from our objective and analyze the downstream evaluation performance and the respective feature covariance matrices on the offline dataset. \cref{fig:ablation} (Center) presents a bar plot of the singular values of the feature covariance matrix ($\Sigma(\phi) := \EE_{s,a\sim \nu} \phi(s,a) \phi(s,a)^{\top}$). \cref{fig:ablation} (Left) shows the downstream OPE performance for features trained with and without the design regularization on the \texttt{Cheetah Run} task. Note that without the regularization, we find that that the feature covariance matrix has much worse condition number, i.e. feature is less exploratory. As our analysis suggests, we also observe a deterioration in evaluation performance without the design regularization to explicitly learn exploratory features.

\textbf{Impact of Linear Bellman Completeness: }
 \cref{fig:ablation} (Right) presents an ablation study where we only optimize for the design term in our objective. We find that downstream OPE performance degrades without directly optimizing for linear BC, suggesting that a feature with good coverage alone is not enough to avoid error amplification. 
 \looseness=-1

\section{Conclusion}

We present~\alg~which leverages rich function approximation to learn Bellman complete and exploratory representations for stable and accurate offline policy evaluation. We provide a mathematical framework of representation learning in offline RL, which generalizes all existing representation learning frameworks from the RL theory literature.
We provide an end-to-end theoretical analysis of our approach for OPE and demonstrate that~\alg~can accurately estimate policy values with polynomial sample complexity. Notably, the complexity has no explicit dependence on the size of the state and action space, instead, it only depends on the statistical complexity of the representation hypothesis class. 
Experimentally, we extensively evaluate our approach on the DeepMind Control Suite, a set of image-based, continuous control, robotic tasks.
First, we show that under the linear evaluation protocol -- using linear regression on top of the representations inside the classic LSPE framework -- our approach outperforms prior RL representation techniques CURL and SPR which leverage contrastive representation learning techniques SimCLR and BYOL respectively.
We also show that~\alg~achieves competitive OPE performance with the state-of-the-art FQE, and noticeably improves upon it when evaluating beyond the initial state distribution. Finally, our ablations show that approximate Linear Bellman Completeness and coverage are crucial ingredients to the success of our algorithm.
Future work includes extending \alg~to offline policy optimization. %

\subsubsection*{Acknowledgements}
This material is based upon work supported by the National
Science Foundation under Grant No. 1846210 and by a Cornell University Fellowship. We thank Rahul Kidambi, Ban Kawas, and the anonymous reviewers for useful discussions and feedback.

\newpage
\bibliographystyle{icml2022}
\bibliography{main}

\newpage

\appendix
\onecolumn

\setlength{\parindent}{0pt}
\setlength{\parskip}{\baselineskip}

\begin{center}\LARGE
\textbf{Appendices}
\end{center}

\section{Metric Entropy and Entropy Integral}
We recall the standard notions of entropy integrals here, based on the following distance on $\Phi$,
\begin{align*}
    d_\Phi(\phi, \wt \phi) &\defeq \Eb[\mD]{\nm{ \phi(s,a) - \wt\phi(s,a) }_2} %
\end{align*}
Let $\mN(t, \Phi)$ denote the $t$-covering number under $d_\Phi$.
\begin{definition}\label{def:phi-g-ent-integral}
Define the entropy integral, which we assume to be finite as,
\begin{align*}
    &\kappa(\Phi) \defeq \int_0^{4} \log^{1/2} \mN(t, \Phi) \diff t
\end{align*}
\end{definition}
When $\Phi$ is finite, $\mN(t) \leq |\Phi|$, so $\kappa(\Phi) \leq \mO(\log^{1/2}(|\Phi|))$.

\section{Technical Lemmas}

\begin{lemma}\label{lm:modified-bernstein}
	Let $X_i$ be i.i.d. random variables s.t. $\abs{X_i} \leq c$ and $\Eb{X_i^2} \leq \nu$, then for any $\delta \in (0, 1)$, we have w.p. at least $1-\delta$,
	\begin{align*}
		\abs{\frac{1}{N} \sum_{i=1}^N X_i - \Eb{X}} \leq \inf_{a > 0} \frac{\nu}{2a} + \frac{(c+a)\log(2/\delta)}{N},
	\end{align*}
	and if $\nu \leq L \Eb{X}$ for some positive $L$, then in particular,
	\begin{align*}
		\abs{\frac{1}{N} \sum_{i=1}^N X_i - \Eb{X}} \leq \frac{1}{2} \Eb{X} + \frac{(c+L)\log(2/\delta)}{N}.
	\end{align*}
\end{lemma}
\begin{proof}
	First, by Bernstein's inequality \citep[Theorem 2.10]{boucheron2013concentration}, we have w.p. $1-\delta$,
	\begin{align*}
		\abs{\frac{1}{N} \sum_{i=1}^N X_i - \Eb{X}} 
		&\leq \sqrt{\frac{2\nu \log(2/\delta)}{N}} + \frac{c\log(2/\delta)}{N}
		\intertext{Using the fact that $2xy \leq x^2/a + ay^2$ for any $a > 0$, split the square root term, }
		&\leq \inf_{a > 0} \frac{\nu}{2a} + \frac{(c+a)\log(2/\delta)}{N}
	\end{align*}
	which yields the first part. 
	If $\nu \leq L \Eb{X}$, picking $a = L$ concludes the proof. 
\end{proof}

We now state several results of Orlicz norms (mostly from \citealp{pollard1990empirical}) for completeness.
For an increasing, convex, positive function $\Phi: \RR^+ \mapsto \RR^+$ , such that $\Phi(x) \in [0, 1)$, define the Orlicz norm as
\begin{align*}
    \nm{ Z }_\Phi := \inf \braces{ C > 0 \mid \Eb{\Phi(|Z|/C)} \leq 1 }.
\end{align*}
It is indeed a norm on the $\Phi$-Orlicz space of random variables $L^\Phi(\nu)$, since it can be interpreted as the Minkowski functional of the convex set $K = \braces{X: \Eb{\Phi(|X|)} \leq 1}$.

Let $x_1, x_2, \dots, x_N$ be i.i.d. datapoints drawn from some underlying distribution.
Let $\omega$ denote the randomness of the $N$ sampled datapoints, and let $\mF_\omega = \braces{ (f(x_i(\omega)))_{i=1}^N: f \in \mF } \subset \RR^N$ denote the (random) set of vectors from the data corresponding to $\omega$.
Let $\bm{\sigma}$ denote a vector of $N$ i.i.d. Rademacher random variables ($\pm 1$ equi-probably), independent of all else.
\begin{lemma}[Symmetrization]\label{lm:symmetrization}
For any increasing, convex, positive function $\Phi$, we have
\begin{align*}
    \Eb[\omega]{ \Phi\prns{ \sup_{\bm{f} \in \mF_\omega} \abs{ \sum_{i=1}^N \bm{f}_i - \EE_\omega \bm{f}_i }   } } \leq \Eb[\omega, \bm{\sigma}]{ \Phi\prns{ 2\sup_{\bm{f} \in \mF_\omega} \abs{\bm{\sigma} \cdot \bm{f}} } }.
\end{align*}
\end{lemma}
\begin{proof}
See Theorem 2.2 of \citep{pollard1990empirical}.
\end{proof}

\begin{lemma}[Contraction]\label{lm:contraction}
Let $\mF \subset \RR^N$, and suppose $\lambda: \RR^N \mapsto \RR^N$ s.t. each component $\lambda_i: \RR \mapsto \RR$ is $L$-Lipschitz.
Then, for any increasing, convex, positive function $\Phi$, we have
\begin{align*}
    \Eb[\bm{\sigma}]{\Phi\prns{ \sup_{\bm{f} \in \mF} \abs{ \bm{\sigma} \cdot \lambda(\bm{f}) } }} \leq \frac{3}{2} \Eb[\bm{\sigma}]{ \Phi\prns{ 2L \sup_{\bm{f} \in \mF} \abs{ \bm{\sigma} \cdot \bm{f} } } }
\end{align*}
\end{lemma}
\begin{proof}
Apply Theorem 5.7 of \citep{pollard1990empirical} to the functions $\lambda_i/L$, which are contractions.
\end{proof}

We now focus on the specific Orlicz space of sub-Gaussian random variables, with the function $\Psi(x) = \frac{1}{5}\exp(x^2)$.
The $\Psi$-Orlicz norm of the maximum of random variables can be bounded by the maximum of the $\Psi$-Orlicz norms.
\begin{lemma}\label{lm:orlicz-maximum}
For any random variables $Z_1, ..., Z_m$, we have
\begin{align*}
    \|\max_{i \leq m} |Z_i| \|_\Psi \leq \sqrt{ 2+\log(m) } \max_{i \leq m} \|Z_i\|_\Psi
\end{align*}
\end{lemma}
\begin{proof}
See Lemma 3.2 of \citep{pollard1990empirical}.
\end{proof}

\begin{lemma}\label{lm:orlicz-psi-subgaussian}
For each $\bm{f} \in \RR^N$, we have $\|\bm{\sigma}\cdot\bm{f}\|_\Psi \leq 2\|\bm{f}\|_2$.
\end{lemma}
\begin{proof}
See Lemma 3.1 of \citep{pollard1990empirical}.
\end{proof}

The following is a truncated chaining result for Orlicz norms.
This result is not new, but many sources state and prove it in terms of covering and Rademacher complexity, rather than for Orlicz norms.
In particular, it generalizes Theorem 3.5 of \citep{pollard1990empirical} -- consider a sequence of $\alpha$'s converging to zero.
\begin{lemma}[Chaining]\label{lm:truncated-chaining}
Let $\mF \subset \RR^N$ such that $0 \in \mF$. Then,
\begin{align*}
    \nm{ \sup_{\bm{f} \in \mF} \abs{ \bm{\sigma}\cdot\bm{f} } }_{\Psi} \leq \inf_{\alpha \geq 0} \braces{ 3\alpha\sqrt{N} + 9 \int_\alpha^{b} \sqrt{ \log(D(\delta/2, \mF)) } \diff\delta },
\end{align*}
where $b = \sup_{\bm{f} \in \mF} \| \bm{f} \|_2$ and $D(\delta, \mF)$ is the Euclidean $\delta$-packing number for $\mF$.
\end{lemma}
\begin{proof}
Suppose $b$ and all the packing numbers are finite, otherwise the right hand side is infinite and there is nothing to show.
For an arbitrary $K > 1$, construct a sequence of $K$ finer and finer approximations to $\mF$,
\begin{align*}
    \{0\} = \mF_0 \subset \mF_1 \subset \dots \subset \mF_K \subset \mF_{K+1} = \mF
\end{align*}
where for any $k \in [K]$, $\mF_k$ satisfies the property that for any $\bm{f} \in \mF$, there exist $n_k(\bm{f}) \in \mF_k$ s.t. $\|n_k(\bm{f}) - \bm{f}\|_2 \leq b2^{-k}$. Indeed this can be done iteratively: for any $\mF_k$, we can construct $\mF_{k+1}$ by adding elements to construct a maximal $b2^{-k}$ packing (maximality ensures the distance requirement, since the existence of any vector which has larger distance can be added to the packing). By definition of $D(\cdot, \mF)$, we have $|\mF_k| \leq D(b2^{-k}, \mF)$.

For any $k \in [K]$, we have by triangle inequality,
\begin{align*}
    \sup_{\bm{f} \in \mF_{k+1}} |\bm{\sigma} \cdot \bm{f}|
    &\leq \sup_{\bm{f} \in \mF_{k+1}} |\bm{\sigma}\cdot n_k(\bm{f}))| + \sup_{\bm{f} \in \mF_{k+1}} |\bm{\sigma}\cdot (n_k(\bm{f}) - \bm{f})|
    \\&= \sup_{\bm{f} \in \mF_k} |\bm{\sigma}\cdot\bm{f}| + \sup_{\bm{f} \in \mF_{k+1}} |\bm{\sigma}\cdot (n_k(\bm{f}) - \bm{f})|
\end{align*}
If $k = K$, we can loosely bound the right-most term by Cauchy-Schwarz, since for any $\bm{f} \in \mF$, we have $|\bm{\sigma}\cdot (n_k(\bm{f}) - \bm{f})| \leq \sqrt{ N } \|n_k(\bm{f})-\bm{f}\|_2 \leq \sqrt{N} b 2^{-K}$. So, we have
\begin{align*}
    \nm{ \sup_{\bm{f} \in \mF} |\bm{\sigma} \cdot \bm{f}| }_\Psi \leq \nm{ \max_{\bm{f} \in \mF_K} |\bm{\sigma}\cdot\bm{f}| }_\Psi + \frac{\sqrt{N} b 2^{-K}}{\sqrt{\log 5}},
\end{align*}
since for any non-negative constant $c$, $\|c\|_\Psi = \inf\braces{C > 0: \frac{1}{5} \exp((c/C)^2) \leq 1} = \frac{c}{\sqrt{\log 5}}$.
If $k < K$, the suprema are taken over finite sets, so the maximum is attained.
Hence, we can apply a special property of the $\Psi$-Orlicz norm (\cref{lm:orlicz-maximum}), to get,
\begin{align*}
    \nm{ \max_{\bm{f} \in \mF_{k+1}} |\bm{\sigma} \cdot \bm{f}| }_\Psi
    &\leq \nm{ \max_{\bm{f} \in \mF_k} |\bm{\sigma}\cdot\bm{f} | }_\Psi + \nm{ \max_{\bm{f} \in \mF_{k+1}} \abs{ \bm{\sigma}\cdot (n_k(\bm{f}) - \bm{f})) } }_\Psi
    \\&\leq \nm{ \max_{\bm{f} \in \mF_k} |\bm{\sigma}\cdot\bm{f} | }_\Psi + \sqrt{2 + \log(|\mF_{k+1}|) } \max_{\bm{f} \in \mF_{k+1}} \nm{\bm{\sigma}\cdot (n_k(\bm{f}) - \bm{f}))}_\Psi
    \intertext{By \cref{lm:orlicz-psi-subgaussian}, }
    &\leq \nm{ \max_{\bm{f} \in \mF_k} |\bm{\sigma}\cdot\bm{f} | }_\Psi + 2\sqrt{2 + \log(|\mF_{k+1}|) } \max_{\bm{f} \in \mF_{k+1}} \nm{ n_k(\bm{f}) - \bm{f})}_2
    \\&\leq \nm{ \max_{\bm{f} \in \mF_k} |\bm{\sigma}\cdot\bm{f} | }_\Psi + 2\sqrt{2 + \log(|\mF_{k+1}|) } \cdot b 2^{-k}.
\end{align*}
Also, note that since $\mF_0 = \{0\}$ by construction, we have $\nm{ \max_{\bm{f} \in \mF_0} |\bm{\sigma}\cdot\bm{f} | }_\Psi = 0$.
Unrolling this, we have
\begin{align*}
    \|\sup_{\bm{f} \subset \mF} |\bm{\sigma}\cdot\bm{f}| \|_\Psi 
    &\leq \frac{\sqrt{N} b 2^{-K}}{\sqrt{\log 5}} + \sum_{k=0}^{K-1} 2^{-k} 2b\sqrt{2 + \log(D(b2^{-(k+1)}, \mF)) }
    \intertext{Since for any $D \geq 2$, we have $\sqrt{2 + \log(1+D)}/\sqrt{D} \leq 2.2$, }
    &\leq \frac{\sqrt{N} b 2^{-K}}{\sqrt{\log 5}} + 4.4b\sum_{k=0}^{K-1} 2^{-k} \sqrt{ \log(D(b2^{-(k+1)}, \mF)) }
    \\&= \frac{\sqrt{N} b 2^{-K}}{\sqrt{\log 5}} + 17.6b\sum_{k=0}^{K-1} (2^{-(k+1)}-2^{-(k+2)}) \sqrt{ \log(D(b2^{-(k+1)}, \mF)) }
    \intertext{Since $D(\cdot, \mF)$ is a monotone decreasing, }
    &\leq \frac{\sqrt{N} b 2^{-K}}{\sqrt{\log 5}} + 17.6\int_{b2^{-(K+1)}}^{b/2} \sqrt{ \log(D(\delta, \mF)) } \diff\delta
    \\&= \frac{\sqrt{N} b 2^{-K}}{\sqrt{\log 5}} + 8.8\int_{b2^{-K}}^{b} \sqrt{ \log(D(\delta/2, \mF)) } \diff\delta.
\end{align*}
Now consider any $\alpha > 0$. Pick $K$ such that $\frac{b}{2^{K+1}} \leq \alpha \leq \frac{b}{2^{K}}$, then we have
\begin{align*}
    \|\sup_{\bm{f} \subset \mF} |\bm{\sigma}\cdot\bm{f}| \|_\Psi \leq \frac{2\alpha \sqrt{N}}{\sqrt{\log 5}} + 8.8 \int_\alpha^{b} \sqrt{ \log(D(\delta/2, \mF)) } \diff\delta.
\end{align*}
Since $\alpha$ was arbitrary, the above bound holds when we take an infimum over $\alpha$.
\end{proof}

\section{Proofs for LSPE}

We first show a generalization of the performance difference lemma (PDL).
\begin{lemma}[Generalized PDL] \label{lm:general-pdl-sarsa}
For any policies $\pi, \pi'$, any function $f:\mathcal{S}\times\mathcal{A}\mapsto \mathbb{R}$, and any initial state distribution $\mu$, we have
\begin{align}
    V^\pi_\mu - \Eb[s \sim \mu]{f(s,\pi')} = \frac{1}{1-\gamma} \Eb[s,a \sim d^\pi_\mu]{ \mT^{\pi'} f (s,a) - f(s,\pi')}
\end{align}
\end{lemma}
\begin{proof}
Let $T^\pi$ be the distribution of trajectories $\tau = (s_0, a_0, s_1, a_1, s_2, a_2, ...)$ from rolling out $\pi$. So, we have,
\begin{align*}
    V^\pi_\mu - \Eb[s \sim \mu]{f(s,\pi')}
    &= \Eb[\tau \sim T^\pi]{\sum_{t=0}^\infty \gamma^t r(s_t, a_t) } - \Eb[s \sim \mu]{f(s,\pi')}
    \\&= \Eb[\tau \sim T^\pi]{\sum_{t=0}^\infty \gamma^t \prns{ r(s_t, a_t) + f(s_t,\pi') - f(s_t,\pi') }  - f(s_0,\pi') } 
    \\&= \Eb[\tau \sim T^\pi]{\sum_{t=0}^\infty \gamma^t \prns{ r(s_t, a_t) + \gamma f(s_{t+1},\pi') - f(s_t,\pi') } } \\
    &= \frac{1}{1-\gamma} \Eb[s, a \sim d^\pi_\mu]{ r(s,a) + \gamma \Eb[s' \sim P(s,a)]{f(s', \pi')} - f(s,\pi') } \\
    &= \frac{1}{1-\gamma} \Eb[s,a \sim d^\pi_\mu]{ \mT^{\pi'} f (s,a) - f(s,\pi')}.
\end{align*}
\end{proof}
This generalizes the PDL, which we can get by setting $f(s,a) = Q^{\pi'}(s,a)$:
\begin{align*}
    V^\pi_\mu - V^{\pi'}_\mu 
    &= \frac{1}{1-\gamma} \Eb[s,a \sim d^\pi_\mu]{ \mT^{\pi'} Q^{\pi'}(s,a) - Q^{\pi'}(s,\pi') } 
    \\&= \frac{1}{1-\gamma} \Eb[s,a \sim d^\pi_\mu]{ Q^{\pi'}(s,a) - V^{\pi'}(s) } 
    \\&= \frac{1}{1-\gamma} \Eb[s,a \sim d^\pi_\mu]{ A^{\pi'}(s,a) }.
\end{align*}
To prove our LSPE guarantee, we'll instantiate $f(s,a)$ to be the estimated $\wh f(s,a)$ from LSPE, and set $\pi = \pi'$. This gives us an expression for the prediction error of LSPE,
\begin{align*}
    \abs{ V_\mu^\pi - \Eb[s \sim \mu]{\wh f_k(s,\pi)} } = \frac{1}{1-\gamma} \abs{ \Eb[d^\pi_\mu]{ \mT^\pi \wh f_k(s,a) - f_k(s,a) } }.
\end{align*}
We then upper bound the right hand side by its $L_2(d^\pi_\mu)$ norm, which is the Bellman error.
Next, we bound the Bellman error of running LSPE by the regression losses at each step.

\begin{lemma}\label{lemma:bellerr}
Consider any policy $\pi$ and functions $f_1,\dots, f_K: \mS \times \mA \mapsto \RR$ that satisfy $\max_{k=1,\dots,K} \|f_k-\mT^\pi f_{k-1}\|_{L_2(d^\pi_{p_0})}\leq\eta$, and $f_0(s,a) = 0$. 
Then, for all $k = 1, \dots, K$, we have $\|f_k-\mT^\pi f_k\|_{L_2(d^\pi_{p_0})}\leq \frac{4}{1-\gamma}\eta+\gamma^{k/2}$.
\end{lemma}
\begin{proof}
For any $k=1,\dots,K$,
\begin{align*}
\|f_k-\mT^\pi f_k\|_{L_2(d^\pi_{p_0})}&\leq 
\|f_k - \mT^\pi f_{k-1}\|_{L_2(d^\pi_{p_0})}+\|\mT^\pi f_{k-1}-\mT^\pi f_k\|_{L_2(d^\pi_{p_0})}
\\&\leq\eta+\gamma (\EE_{s,a\sim d^\pi_{p_0}}[(\Eb[s',a' \sim P(s,a)\circ\pi]{f_{k-1}(s',a') - f_k(s',a')})^2])^{1/2}
\\&\leq\eta+\gamma (\EE_{s,a\sim d^\pi_{p_0}, s',a' \sim P(s,a)\circ\pi}[(f_{k-1}(s',a')-f_k(s',a'))^2])^{1/2}
\intertext{Since $d^\pi_{p_0}(s,a) = \gamma \EE_{\tilde s,\tilde a \sim d^{\pi}_{P_0}} P(s | \tilde s,\tilde a) \pi(a | s)  + (1-\gamma) p_0(s )\pi(a|s)$, and the quantity inside the expectation is a square, and thus non-negative, }
&\leq \eta+\gamma (\gamma^{-1}\EE_{s,a\sim d^\pi_{p_0}}[(f_{k-1}(s,a)-f_k(s,a))^2])^{1/2}
\\&=\eta+\sqrt{\gamma} \|f_{k-1}-f_k\|_{L_2(d^\pi_{p_0})}
\\&\leq\eta+\sqrt{\gamma} \prns{\eta+\|f_{k-1}-\mT^\pi f_{k-1}\|_{L_2(d^\pi_{p_0})}}.
\end{align*}
Unrolling the recursion and using $\|f_0-\mT^\pi f_0\|_{L_2(d^\pi_{p_0})}\leq 1$, we have
\begin{align*}
\|f_K-\mT^\pi f_K\|_{L_2(d^\pi_{p_0})}
&\leq \eta + \sqrt{\gamma}(\eta + \sqrt{\gamma}(\eta + \dots \sqrt{\gamma}(\eta + 1) \dots ))
\\&= \frac{1-\sqrt{\gamma}^K}{1-\sqrt{\gamma}}\eta+\gamma^{K/2},
\end{align*}
which gives the claim since $\frac{1}{1-\sqrt{\gamma}}\leq2(1-\gamma)^{-1}$ and $1-\sqrt{\gamma}^K\leq 1$.
\end{proof}

The following lemma is a ``fast rates''-like result for norms. It shows that the norm induced by the empirical covariance matrix $\wh\Sigma$ can be bounded by the norm induced by two times the population covariance matrix $\Sigma$, up to some $\wt\mO(N^{-1/2})$ terms.
\begin{lemma}[Fast Rates for $\Sigma$-norm]\label{lm:switching-emp-and-pop-sigma-norm}
For any $\delta \in (0, 1)$, we have with probability at least $1-\delta$, for any $x \in B_W$, we have
\begin{align*}
    \|x\|_{\wh\Sigma} \leq 2\|x\|_{\Sigma} + 5W\sqrt{\frac{d\log(N/\delta)}{N}},
\end{align*}
and
\begin{align*}
    \|x\|_{\Sigma} \leq 2\|x\|_{\wh\Sigma} + 5W\sqrt{\frac{d\log(N/\delta)}{N}}.
\end{align*}
\end{lemma}
\begin{proof}
First, fix any $x \in B_W$. Since $(x\tr\phi(s,a))^2 \leq W^2$ almost surely, we have $\Eb{(x\tr \phi(s,a))^4} \leq W^2 \Eb{(x\tr \phi(s,a))^2}$. So by \cref{lm:modified-bernstein}, w.p. at least $1-\delta$,
\begin{align*}
    \abs{ \frac{1}{N}\sum_{i=1}^N (x\tr\phi(s_i,a_i))^2 - \Eb{(x\tr \phi(s,a))^2} } \leq \frac{1}{2}\Eb{(x\tr \phi(s,a))^2} + \frac{2W^2\log(2/\delta)}{N}.
\end{align*}
In particular, this means $\|x\|_{\wh\Sigma}^2 \leq \frac{3}{2} \|x\|_{\Sigma}^2 + \frac{2W^2 \log(2/\delta)}{N}$ and $\|x\|_\Sigma^2 \leq 2\|x\|_{\wh\Sigma}^2 + \frac{4W^2 \log(2/\delta)}{N}$.
Now union bound over a $W/\sqrt{N}$-net of $B_W$, which can be done with $(1+2\sqrt{N})^d$ elements.
The approximation error from this cover is
\begin{align*}
    \|x\|_{\wh\Sigma} 
    &\leq \|n(x)\|_{\wh\Sigma} + \|n(x) - x\|_{\wh\Sigma} 
    \\&\leq \sqrt{\frac32}\|n(x)\|_\Sigma + \sqrt{\frac{2W^2\log(2(1+2\sqrt{N})^d/\delta)}{N}} + \frac{W}{\sqrt{N}}
    \\&\leq \sqrt{\frac32}\|x\|_\Sigma + \sqrt{\frac32}\|n(x)-x\|_\Sigma + \sqrt{\frac{2W^2\log(2(1+2\sqrt{N})^d/\delta)}{N}} + \frac{W}{\sqrt{N}}
    \\&\leq \sqrt{\frac32}\|x\|_\Sigma + 4W\sqrt{\frac{d\log(N/\delta)}{N}},
\end{align*}
where $n(x)$ is the closest element in the net to $x$.
Similarly,
\begin{align*}
    \|x\|_\Sigma
    &\leq \|n(x)\|_{\Sigma} + \|n(x)-x\|_\Sigma
    \\&\leq 2\|n(x)\|_{\wh\Sigma} + \sqrt{\frac{4W^2 \log(2(1+2\sqrt{N})^d/\delta)}{N}} + \frac{W}{\sqrt{N}}
    \\&\leq 2\|x\|_{\wh\Sigma} + 2\|n(x)-x\|_{\wh\Sigma} + + \sqrt{\frac{4W^2 \log(2(1+2\sqrt{N})^d/\delta)}{N}} + \frac{W}{\sqrt{N}}
    \\&\leq 2\|x\|_{\wh\Sigma} + 5W\sqrt{\frac{d\log(N/\delta)}{N}}.
\end{align*}
\end{proof}

Now define the following notation for analyzing Least Squares Policy Evaluation (LSPE).
For every target vector $\|\vartheta\|_2\leq W$, define 
\begin{align*}
y^\vartheta&=r+\gamma \vartheta\tr\phi(s',\pi),
\\\theta_\vartheta&=\argmin_{\|\theta\|\leq W}\EE_{s,a \sim \nu, s' \sim P(s,a)}[(y^\vartheta-\theta\tr\phi(s,a))^2] := \ell(\theta, \vartheta),
\\y^\vartheta_i&=r_i+\gamma \vartheta\tr\phi(s_i',\pi),
\\\wh\theta_\vartheta&=\argmin_{\|\theta\|\leq W}\frac{1}{N}\sum_{i=1}^N(y^\vartheta_i-\theta\tr\phi(s_i,a_i))^2 := \wh\ell(\theta,\vartheta).
\end{align*}

The following lemma are useful facts about the optimal $\theta_\vartheta$ and $\wh\theta_\vartheta$.
\begin{lemma}\label{lm:first-order-optimality-theta-vartheta}
For any $\vartheta_1, \vartheta_2$, we have
\begin{align*}
    &\|\theta_{\vartheta_1} - \theta_{\vartheta_2}\|_\Sigma \leq \sqrt{2\gamma W \|\vartheta_1 - \vartheta_2\|_2 },
    \\&\|\wh\theta_{\vartheta_1} - \wh\theta_{\vartheta_2}\|_{\wh\Sigma} \leq \sqrt{2\gamma W \|\vartheta_1 - \vartheta_2\|_2 }.
\end{align*}
For any $\theta,\vartheta$, we have
\begin{align*}
    \ell(\theta,\vartheta) - \ell(\theta_\vartheta,\vartheta) \geq \|\theta-\theta_\vartheta\|_\Sigma^2.
\end{align*}
\end{lemma}
\begin{proof}
First, recall that $\theta_\vartheta$ minimizes $f(\theta) = \Eb{(y^\vartheta - \theta\tr \phi(s,a))^2}$, which has the Jacobian $\nabla f(\theta) = 2\Eb{(\theta\tr \phi(s,a) - y^\vartheta)\phi(s,a)}$. Since $\theta_\vartheta$ is optimal over $B_W$, the necessary optimality condition is that $-\nabla f(\theta_\vartheta) \in N_{B_W}(\theta_\vartheta)$, the normal cone of $B_W$ at $\theta_\vartheta$, i.e. for any $\theta \in B_W$, we have $\langle \Eb{\phi(s,a)(y^\vartheta - \theta_\vartheta\tr \phi(s,a))}, \theta-\theta_\vartheta \rangle \leq 0$. In particular, we have
\begin{align*}
    &\langle \Eb{\phi(s,a)(y^{\vartheta_1} - \theta_{\vartheta_1}\tr \phi(s,a))}, \theta_{\vartheta_2}-\theta_{\vartheta_1} \rangle \leq 0
    \\& \langle \Eb{\phi(s,a)(y^{\vartheta_2} - \theta_{\vartheta_2}\tr \phi(s,a))}, \theta_{\vartheta_1}-\theta_{\vartheta_2} \rangle \leq 0
\end{align*}
adding the two we get
\begin{align*}
    \|\theta_{\vartheta_1} - \theta_{\vartheta_2}\|_\Sigma^2 
    &= \langle \Eb{\phi(s,a)(\theta_{\vartheta_1}-\theta_{\vartheta_2})\tr \phi(s,a))}, \theta_{\vartheta_1}-\theta_{\vartheta_2} \rangle 
    \\&\leq \langle \Eb{(y^{\vartheta_1} - y^{\vartheta_2})\phi(s,a)}, \theta_{\vartheta_1}-\theta_{\vartheta_2} \rangle
    \\&=  \gamma \langle \Eb{\phi(s,a) \phi(s',\pi)\tr } (\vartheta_1 - \vartheta_2), \theta_{\vartheta_1}-\theta_{\vartheta_2} \rangle
    \\&\leq  \gamma \| \Eb{\phi(s,a) \phi(s',\pi)\tr } \|_2 \| \vartheta_1 - \vartheta_2 \|_2  \| \theta_{\vartheta_1}-\theta_{\vartheta_2} \|_2
    \\&\leq 2\gamma W \|\vartheta_1 - \vartheta_2\|_2.
\end{align*}
The claim with $\wh\theta_{\vartheta_1},\wh\theta_{\vartheta_2}$ follows by the same argument.

For the second claim, we first apply the Parallelogram law, followed by the first-order optimality of $\theta_\vartheta$,
\begin{align*}
\ell(\theta,\vartheta)&=\ell(\theta_\vartheta,\vartheta)+\EE_\nu((\theta_\vartheta-\theta)\tr\phi(s,a))^2+2\EE_\nu[(y^\vartheta-\theta\tr\phi(s,a))\phi(s,a)\tr(\theta_\vartheta-\theta)]
\\&\geq\ell(\theta_\vartheta,\vartheta)+\EE_\nu((\theta_\vartheta-\theta)\tr\phi(s,a))^2+0
\\&=\ell(\theta_\vartheta,\vartheta) + \|\theta-\theta_\vartheta\|_\Sigma^2.
\end{align*}
\end{proof}

Now we show our key lemma about the concentration of least squares, uniformly over all targets generated by $\vartheta \in B_W$.
\begin{lemma}[Concentration for Least Squares] \label{lemma:concentrate}
Let $\mF=\{(s,a)\mapsto \theta\tr\phi(s,a):\|\theta\|_2 \leq W\}$, with $\|\phi(s,a)\|_2 \leq 1$.
Then, for any $\delta \in (0, 1)$ w.p. at least $1-\delta$, we have
\begin{align*}
\sup_{\vartheta \in B_W} \|\hat\theta_\vartheta-\theta_\vartheta\|_\Sigma < 120d(1+W) \frac{\log(N)\sqrt{\log(10/\delta)}}{\sqrt{N}}.
\end{align*}
\end{lemma}
\begin{proof}
By \cref{lm:switching-emp-and-pop-sigma-norm}, we have that w.p. at least $1-\delta$, we can bound the random $\|\cdot\|_{\wh\Sigma}$ norm by $\|\cdot\|_{\Sigma}$, and vice versa, simultaneously over all vectors in a ball,
\begin{align*}
    \mE = \{\forall x \in B_{2W}, \|x\|_{\wh\Sigma} \leq 2\|x\|_\Sigma + t \text{ and } \|x\|_{\Sigma} \leq 2\|x\|_{\wh\Sigma} + t\},
\end{align*}
provided $t \geq 5W\sqrt{\frac{d\log(N/\delta)}{N}}$. For the remainder of this proof, we \emph{condition} on this high-probability event. That is, any probability and expectations will be \emph{implicitly conditioned} on $\mE$.

First, we'll show that for an arbitrary and fixed $\vartheta \in B_W$, w.p. at least $1-\delta$, we have
\begin{align*}
    \|\wh\theta_\vartheta - \theta_\vartheta\|_\Sigma < t.
\end{align*}
To do so, we will bound the probability of the complement, which in turn can be simplified by the following chain of arguments,
\begin{align*}
    & \|\hat\theta_\vartheta-\theta_\vartheta \|_\Sigma \geq t
    \intertext{By $\ell(\theta,\vartheta) - \ell(\theta_\vartheta,\vartheta) \geq \|\theta-\theta_\vartheta\|_\Sigma$  from \cref{lm:first-order-optimality-theta-vartheta}, }
    &\implies \ell(\hat\theta_\vartheta,\vartheta)-\ell(\theta_\vartheta,\vartheta)\geq t^2
    \\&\implies \exists \|\theta\|\leq W~:~\ell(\theta,\vartheta)-\ell(\theta_\vartheta,\vartheta)\geq t^2, \hat\ell(\theta,\vartheta)-\hat\ell(\theta_\vartheta,\vartheta) \leq 0
    \intertext{By convexity and continuity of $\ell(\cdot, \vartheta)$, we can make this strict equality.
    Indeed, given this, by Intermediate Value Theorem, there exists $\lambda\in[0,1]$ such that $\theta'=(1-\lambda)\theta+\lambda \theta_\vartheta$ has $\ell(\theta',\vartheta)-\ell(\theta_\vartheta,\vartheta)=\nu$.
    Then by convexity, $\hat\ell(\theta',\vartheta)-\hat\ell(\theta_\vartheta,\vartheta)\leq (1-\lambda)\hat\ell(\theta,\vartheta)-(1-\lambda)\hat\ell(\theta_\vartheta,\vartheta)\leq0$.}
    &\implies \exists \|\theta\|\leq W~:~\ell(\theta,\vartheta)-\ell(\theta_\vartheta,\vartheta)= t^2, \hat\ell(\theta,\vartheta)-\hat\ell(\theta_\vartheta,\vartheta) \leq 0
    \\&\implies \exists \|\theta\|\leq W~:~\|\theta-\theta_\vartheta\|_\Sigma \leq t, (\ell(\theta,\vartheta)-\hat\ell(\theta,\vartheta))-(\ell(\theta_\vartheta,\vartheta)-\hat\ell(\theta_\vartheta,\vartheta)) \geq t^2
    \intertext{By conditioning on $\mE$, }
    &\implies \exists \|\theta\|\leq W~:~\|\theta-\theta_\vartheta\|_{\wh\Sigma} \leq 3t, (\ell(\theta,\vartheta)-\hat\ell(\theta,\vartheta))-(\ell(\theta_\vartheta,\vartheta)-\hat\ell(\theta_\vartheta,\vartheta)) \geq t^2
\end{align*}
Hence, we now focus on bounding
\begin{align*}
\PP\prns{ \sup_{\theta \in \Theta} \abs{ \sum_{i=1}^N \bm{\zeta}_i(\theta,\vartheta) - \bm{\zeta}_i(\theta_\vartheta,\vartheta)  - \Eb[\nu]{ \bm{\zeta}_i(\theta,\vartheta) - \bm{\zeta}_i(\theta_\vartheta,\vartheta) } } \geq Nt^2 } \leq \delta,
\end{align*}
where we define
\begin{align}
\zeta_i(\theta,\vartheta) &= (y_i^\vartheta - \theta\tr\phi(s_i,a_i))^2, \nonumber
\\ \Theta&=\{ \theta~:~\,\|\theta\|\leq W,\,\|\theta-\theta_\vartheta\|_{\wh\Sigma} \leq 3t\}. \label{eq:def-theta-localization}
\end{align}
Since $\zeta_i(\theta,\vartheta) - \zeta_i(\theta_\vartheta,\theta) = \lambda_i( (\theta-\theta_\vartheta)\tr \phi(s_i,a_i) )$ where $\lambda_i(x) = x (2y_i^\vartheta - (\theta+\theta_\vartheta)\tr \phi(s_i,a_i))$ is $2(1+W)$-Lipschitz and $\lambda_i(0) = 0$, the contraction lemma (\cref{lm:contraction}) tells us that it suffices to consider a simpler class. 
Define
\begin{align*}
    \omega_i(\theta,\vartheta) &= (\theta-\theta_\vartheta)\tr\phi(s_i,a_i)
\end{align*}
and $\bm{\omega}(\theta,\vartheta) = (\omega_i(\theta,\vartheta))_{i=1}^N$.
By Chaining (\cref{lm:truncated-chaining}), we have
\begin{align*}
    &\Eb[\bm{\sigma}]{ \Psi\prns{\frac1J\sup_{\Theta}\abs{ \bm{\sigma}\cdot\bm{\omega}(\theta,\vartheta) } } }\leq1 
    \\\text{where }\; &J= \inf_{\alpha \geq 0} \braces{ 3\alpha\sqrt{N} + 9 \int_\alpha^{b} \sqrt{ \log(D(\delta/2, \bm{\omega}(\Theta) )) } \diff\delta },
    \\&\bm{\omega}(\Theta) = \braces{\bm{\omega}(\theta,\vartheta)~:~\theta \in \Theta},
\end{align*}
$D(\delta, \mF)$ is the Euclidean packing number of $\mF\subset\RR^N$, and $b = \sup_{\Theta} \|\bm{\omega}(\theta,\vartheta)\|_2$ is the envelope.

Now we'll bound the truncated entropy integral, $J$.
First notice that $b \leq 3t\sqrt{N}$, based on the definition of $\Theta$ (\cref{eq:def-theta-localization}), which localizes in $\|\cdot\|_{\wh\Sigma}$,
\begin{align*}
    \frac{b}{\sqrt{N}} = \sup_\Theta \sqrt{ \frac{1}{N} \sum_{i=1 }^N (\theta-\theta_\vartheta)\tr\phi(s_i,a_i)\phi(s_i,a_i)\tr(\theta-\theta_\vartheta) }
    = \sup_{\Theta} \|\theta-\theta_\vartheta\|_{\wh\Sigma} 
    \leq 3t.
\end{align*}

Now, we bound the packing number $D(\cdot, \bm{\omega}(\Theta))$. Let $\theta_1,\theta_2\in B_W$ be arbitrary,
\begin{align*}
    \frac{\|\bm{\omega}(\theta_1,\vartheta) - \bm{\omega}(\theta_2,\vartheta)\|_2}{\sqrt{N}}
    = \|\theta_1 - \theta_2\|_{\wh\Sigma}
    \leq \sqrt{\sigma_{max}(\wh\Sigma)} \|\theta_1-\theta_2\|_2
    \leq \|\theta_1-\theta_2\|_2.
\end{align*}
So for any $\eps$, we can construct an $\eps$-cover by setting $\|\theta_1-\theta_2\|_2 \leq \eps/\sqrt{N}$, which requires $(W(1+2\sqrt{N}/\eps))^d$ points.
Let $N(\eps, \mF)$ denote the Euclidean covering number of $\mF \subset \RR^N$.
Then, 
\begin{align*}
    \log D(\eps, \bm{\omega}(\Theta)) 
    \leq \log N(\eps/2, \bm{\omega}(\Theta))
    \leq d \log W(1+4\sqrt{N}/\eps)
\end{align*}
So,
\begin{align*}
    J &\leq \int_{0}^{3t\sqrt{N}} \sqrt{\log(D(\eps/2, \bm{\omega}(\Theta)))} \diff\eps
    \\&\leq \int_{0}^{3t\sqrt{N}} \sqrt{d\log(W(1+8\sqrt{N}/\eps))} \diff\eps
    \\&\leq 3t\sqrt{dN} (\sqrt{\log W} + \int_0^1\sqrt{\log(1+3/(\eps t)}) ) \diff\eps
    \\&\leq 3t\sqrt{dN}(\sqrt{\log W} + \sqrt{\log(4/t)}),
\end{align*}
since $\int_0^1\sqrt{\log(1+c/\eps)}\diff\eps \leq \sqrt{\log(1+c)}$ for any $c > 0$, and assuming $t \leq 1$.

Now we put everything together. Let $c$ denote a positive constant, 
\begin{align*}
    &\PP\prns{ \sup_{\Theta} \abs{ \sum_{i=1}^N \bm{\zeta}_i(\theta,\vartheta) - \bm{\zeta}_i(\theta_\vartheta,\vartheta)  - \Eb[\nu]{ \bm{\zeta}_i(\theta,\vartheta) - \bm{\zeta}_i(\theta_\vartheta,\vartheta) } } \geq Nt^2 }
    \intertext{Since $\Psi$ is increasing, }
    &=\PP\prns{ \Psi\prns{c \sup_{\Theta} \abs{ \sum_{i=1}^N \bm{\zeta}_i(\theta,\vartheta) - \bm{\zeta}_i(\theta_\vartheta,\vartheta)  - \Eb[\nu]{ \bm{\zeta}_i(\theta,\vartheta) - \bm{\zeta}_i(\theta_\vartheta,\vartheta) } } }  \geq \Psi\prns{ cNt^2 } }
    \intertext{By Markov's inequality, }
    &\leq \frac{\Eb{\Psi\prns{ c \sup_{\Theta} \abs{ \sum_{i=1}^N \bm{\zeta}_i(\theta,\vartheta) - \bm{\zeta}_i(\theta_\vartheta,\vartheta)  - \Eb[\nu]{ \bm{\zeta}_i(\theta,\vartheta) - \bm{\zeta}_i(\theta_\vartheta,\vartheta) } }}}}{\Psi\prns{ c Nt^2 }}
    \intertext{By Symmetrization (\cref{lm:symmetrization}), }
    &\leq \frac{\Eb{\Eb[\bm{\sigma}]{\Psi\prns{ 2c \sup_{\Theta} \abs{ \bm{\sigma}\cdot (\bm{\zeta}(\theta,\vartheta) - \bm{\zeta}(\theta_\vartheta,\vartheta)) } }}}}{\Psi\prns{ c N t^2 }}
    \intertext{By Contraction (\cref{lm:contraction}) and that $\bm{\zeta}_i(\theta,\vartheta) - \bm{\zeta}_i(\theta_\vartheta,\vartheta) = \lambda_i( \bm{\omega}_i(\theta,\vartheta) )$ where $\lambda_i$ is $L=2(1+W)$ Lipschitz, }
    &\leq \frac{3}{2}\cdot \frac{ \Eb{\Eb[\bm{\sigma}]{\Psi\prns{ 4cL \sup_{\Theta} \abs{ \bm{\sigma}\cdot \bm{\omega}(\theta,\vartheta) } }}}}{\Psi\prns{ c N t^2 }}
    \intertext{Setting $c= \frac{1}{4LJ}$ and applying Chaining (\cref{lm:truncated-chaining}), the numerator is bounded by 1, }
    &\leq 8 \exp\prns{- \prns{ \frac{Nt^2}{4LJ} }^2  }
    \intertext{Applying upper bound on $J$, }
    &\leq 8 \exp\prns{- \prns{ \frac{Nt^2}{8(1+W) \cdot 3t\sqrt{dN} (\sqrt{\log W} + \sqrt{\log (4/t)}) } }^2 }
    \\&\leq 8 \exp\prns{- \frac{Nt^2}{24^2(1+W)^2 d (\log W + \log (4/t)) } }
    \intertext{Now, we set $t = 24(1+W) \sqrt{\frac{d\log(N)\log(1/\delta)}{N}}$,}
    &\leq 8 \exp\prns{- \frac{\log(N)\log(1/\delta)}{\log(N/d\log(N)) } }
    \intertext{Since $\log(N) \geq \log(N/d\log(N))$,}
    &\leq 8\delta.
\end{align*}

Hence, we have shown that for an arbitrary and fixed $\vartheta \in B_W$, w.p. $1-9\delta$, we have
\begin{align*}
    \|\wh\theta_\vartheta - \theta_\vartheta\|_\Sigma < t = 24(1+W) \sqrt{\frac{d\log(N)\log(1/\delta)}{N}}.
\end{align*}

We finally apply a union bound over $\vartheta \in B_W$. 
Consider a $W/N$-cover of $\vartheta \in B_W$, which requires $(1+2N)^d$ points.
Then, for any $\vartheta \in B_W$, we have
\begin{align*}
    \|\wh\theta_\vartheta - \theta_\vartheta\|_\Sigma
    &\leq \|\wh\theta_\vartheta - \wh\theta_{n(\vartheta)}\|_\Sigma
    + \|\wh\theta_{n(\vartheta)} - \theta_{n(\vartheta)}\|_\Sigma
    + \|\theta_{n(\vartheta)} - \theta_\vartheta\|_\Sigma
    \\&\leq (2\|\wh\theta_\vartheta - \wh\theta_{n(\vartheta)}\|_{\wh\Sigma} + t) + t + \sqrt{2\gamma W \|n(\vartheta) - \vartheta\|}
    \\&\leq 2t + 3\sqrt{2\gamma W \|n(\vartheta) - \vartheta\|}
    \\&\leq 2t + 3\sqrt{2\gamma W^2/N}
    \\&\leq 5t.
\end{align*}
Thus, we have shown that w.p. $1-\delta$,
\begin{align*}
    \sup_{\vartheta \in B_W} \|\wh\theta_\vartheta - \theta_\vartheta\|_\Sigma < 120d(1+W) \frac{\log(N)\sqrt{\log(10/\delta)}}{\sqrt{N}}.
\end{align*}
\end{proof}

A nice corollary is that when $\Sigma$ provides full coverage, i.e. $\Sigma \succeq \beta I$ for some positive $\beta$, then we can bound
\begin{align*}
    \sup_{\vartheta \in B_W} \|\hat\theta_\vartheta-\theta_\vartheta\|_2 \leq \sigma_{min}(\Sigma)^{-1/2} \sup_{\vartheta \in B_W} \|\wh\theta_\vartheta - \theta_\vartheta\|_\Sigma \leq \beta^{-1/2} \cdot \sup_{\vartheta \in B_W} \|\wh\theta_\vartheta - \theta_\vartheta\|_\Sigma.
\end{align*}

\subsection{Main LSPE theorem} \label{sec:proof-fqe}

We now prove our LSPE sample complexity guarantee \cref{thm:fqe}.
\fqesamplecomplexity*

\begin{proof}[Proof of \cref{thm:fqe}]
By Lemmas~\ref{lm:general-pdl-sarsa} and \ref{lemma:bellerr},
$$
\abs{V_{p_0}^\pi - \Eb[s \sim p_0]{\wh f_k(s,\pi)}} \leq
\frac{4}{(1-\gamma)^2}\max_{k=1,2,\dots}\|\hat f_k-\mT^\pi \hat f_{k-1}\|_{L_2(d^\pi_{p_0})}
+\frac{\gamma^{k/2}}{1-\gamma}.
$$

Next, we bound the maximum regression error. Consider any initial state distribution $p_0$, then we have
\begin{align*}
\max_{k=1,2,\dots}\|\hat f_k-\mT^\pi \hat f_{k-1}\|_{L_2(d^\pi_{p_0})}
&\leq \sup_{\|\vartheta\|\leq W}
\|\hat \theta_{\vartheta}\tr\phi-\mT^\pi(\vartheta\tr\phi)\|_{L_2(d^\pi_{p_0})}\\
&\leq \sup_{\|\vartheta\|\leq W}
\|\hat \theta_{\vartheta}\tr\phi-\theta_\vartheta\tr\phi\|_{L_2(d^\pi_{p_0})}
+
\sup_{\|\vartheta\|\leq W}
\|\theta_{\vartheta}\tr\phi-\mT^\pi(\vartheta\tr\phi)\|_{L_2(d^\pi_{p_0})}
\\&\leq \sqrt{\kappa(p_0)} \sup_{\|\vartheta\|\leq W} \|\hat \theta_{\vartheta}-\theta_\vartheta\|_\Sigma + \sqrt{\nm{\frac{\diff d_{p_0}^\pi}{\diff\nu}}_\infty}\sup_{\|\vartheta\|\leq W} \|\theta_{\vartheta}\tr\phi-\mT^\pi(\vartheta\tr\phi)\|_{L_2(\nu)}
\\ &\leq \sqrt{\kappa(p_0)} \sup_{\vartheta \in B_W} \|\hat \theta_{\vartheta}-\theta_\vartheta\|_\Sigma + \sqrt{\nm{\frac{\diff d_{p_0}^\pi}{\diff\nu}}_\infty}\eps_\nu.
\end{align*}
The quantity $\sup_{\vartheta \in B_W} \|\hat \theta_{\vartheta}-\theta_\vartheta\|_\Sigma$ can be directly bounded by \cref{lemma:concentrate} w.p. at least $1-\delta$.
Thus, we have shown the desired result: for any initial state distribution $p_0$, 
\begin{align*}
    \abs{V_{p_0}^\pi - \Eb[s,a \sim p_0 \circ \pi]{\wh f_k(s,a)}} \leq
\frac{4}{(1-\gamma)^2} \prns{\sqrt{\kappa(p_0)}120d(1+W) \frac{\log(N)\sqrt{\log(10/\delta)}}{\sqrt{N}} + \sqrt{\nm{\frac{\diff d_{p_0}^\pi}{\diff\nu}}_\infty}} \eps_\nu +\frac{\gamma^{k/2}}{1-\gamma}.
\end{align*}
\end{proof}

\section{Proofs for Linear BC Equivalence} \label{sec:equivalent_lbc}

\equivalentlbc*

\begin{proof}
$(\Longleftarrow)$ Suppose that $\|M\|_2 < 1$ and $\rho \in B_W$ satisfy,
\begin{align*}
\mathbb{E}_{\nu} \left\| \begin{bmatrix}M \\ \rho\tr \end{bmatrix} \phi(s,a) -  \begin{bmatrix}  \gamma \EE_{s'\sim P(s,a)}\phi(s', \pi_e) \\ r(s,a) \end{bmatrix}\right\|^2_2 = 0.
\end{align*}
Then, for any $w_1 \in B_{W}$, setting $w_2 = \rho + M\tr w_1$ satisfies,
\begin{align*}
\|w_2\tr\phi - \mT^{\pi_e}(w_1\tr\phi)\|_\nu^2
&= \EE_{\nu} \prns{  w_2\tr \phi(s,a) - r(s,a) - \gamma   \Eb[s'\sim P(s,a)]{ w_1\tr \phi(s',\pi_e) } }^2 
\\&= \EE_{\nu} \prns{ \rho\tr\phi(s,a) - r(s,a) + w_1\tr M\phi(s,a) - \gamma   \Eb[s'\sim P(s,a)]{ w_1\tr \phi(s',\pi_e) } }^2
\\&= \EE_{\nu} \prns{  w_1\tr \prns{ M\phi(s,a) - \gamma   \Eb[s'\sim P(s,a)]{ \phi(s',\pi_e) } } }^2
\\&=0.
\end{align*} 
Also, we have $\|w_2\|_2 \leq \|\rho\|_2 + \|M\|_2\|w_1\|_2 \leq W$ since $W \geq \frac{\nm{\rho}_2}{1-\nm{M}_2}$. Thus, $\phi$ satisfies exact Linear BC. \hfill \break

$(\Longrightarrow)$
Suppose $\phi$ satisfies exact Linear BC, that is
$$
    \max_{w_1 \in B_W}\min_{w_2 \in B_W} \|w_2\tr \phi - \mT^\pi(w_1\tr\phi)\|_\nu^2 = 0.
$$
To see that there exists $\rho\in B_W$ that linearizes the reward w.r.t $\phi$ under $\nu$, set $w_1 = 0$, and we have:
\begin{align*}
\min_{w_2 \in B_W}  \EE_{s,a\sim \nu} \left\|  w_2^{\top} \phi(s,a) - r(s,a)  \right\|_2^2 = 0.
\end{align*}
Let $\rho$ to be the minimizer of the above objective.

Now we need to show that there exists a $M\in\RR^{d\times d}$ with $\|M\|_2<\sqrt{1-\frac{\|\rho\|^2}{W^2}}$ that satisfies 
$$
\EE_{s,a\sim \nu} \left\| M \phi(s,a) - \gamma \mathbb{E}_{s'\sim P(s,a)} \phi(s',\pi_e) \right\|_2^2 = 0.
$$
To extract the $i$-th row of $M$, plug in $w_i = W e_i$ (note that $w_i \in B_W$). By exact Linear BC, we know that there exists a vector $v_i \in B_W$, such that:
\begin{align*}
\left\| v_i\tr \phi(s,a) -  \rho\tr\phi(s,a) - \gamma W \EE_{s'\sim P(s,a)} e_i\tr\phi(s',\pi_e) \right\|_\nu = 0.
\end{align*}
Repeating this for every $i\in[d]$, we can construct $M$ as follows,
\begin{align*}
    M = \frac{1}{W} \begin{bmatrix}
        (v_1 - \rho)\tr \\
        \vdots \\
        (v_d - \rho)\tr
    \end{bmatrix},
\end{align*}
which satisfies,
\begin{align*}
&\EE_{s,a\sim \nu} \left\| M \phi(s,a) - \gamma \mathbb{E}_{s'\sim P(s,a)} \phi(s',\pi_e) \right\|_2^2
\\&=\sum_{i=1}^d \EE_{s,a\sim \nu} \left\| e_i\tr \prns{ M \phi(s,a) - \gamma \mathbb{E}_{s'\sim P(s,a)} \phi(s',\pi_e) } \right\|_2^2
\\&= \sum_{i=1}^d \EE_{s,a\sim \nu} \left\| \frac1W (v_i-\rho)\tr\phi(s,a) - \gamma \mathbb{E}_{s'\sim P(s,a)} e_i\tr \phi(s',\pi_e) \right\|_2^2 = 0.
\end{align*}
Hence, we have $\left\| M \phi(s,a) - \gamma \mathbb{E}_{s'\sim P(s,a)} \phi(s',\pi_e) \right\|_\nu^2 = 0$.

Finally, we must show that $\|M\|_2 < \sqrt{1-\frac{\|\rho\|^2}{W^2}}$.
First we show that $\nm{M}_2 \leq 1$. 
For any $w_1 \in B_W$, by exact linear BC, there exists $w_2 \in B_W$ s.t. $\nm{w_2\tr \phi(s,a) - r(s,a) - \gamma \Eb[s' \sim P(s,a)]{w_1\tr \phi(s',\pi_e)}}_\nu^2 = 0$, and by the construction of $M$, satisfies $\nm{ \prns{ w_2-\rho- M\tr w_1}\tr \phi(s,a) }_\nu^2 = \nm{ w_2-\rho- M\tr w_1 }_{\Sigma(\phi)}^2 = 0$. Since $\Sigma(\phi)$ is positive definite, we have that $w_2 = \rho + M\tr w_1$ is the unique choice of $w_2$, which by exact linear BC is in $B_W$.
Hence, we've shown that for any $w_1 \in B_W$, we also have that $\rho + M\tr w_1 \in B_W$. Now take $w_1$ and $-w_1$, subtracting the two expressions yields that $2 M^{\star\intercal} w_1 \in B_{2W}$. Since this is true for arbitrary $w_1$, taking supremum over $w_1 \in B_W$ shows that $\|M\|_2 \leq 1$.

Now we show that the inequality must be strict.
Consider the singular value decomposition: $M = \sum_{i=1}^d \sigma_i u_i v_i\tr$ where $\braces{u_i}_{i\in[d]}$ and $\braces{v_i}_{i\in[d]}$ are each an orthonormal basis of $\RR^d$, and $\sigma_i$ is the $i$-th largest singular value. Without loss of generality, suppose $\rho\tr u_1 \geq 0$, since we can always flip the sign of $u_1$.
If we pick $x = W v_1 \in B_W$, we have $M x = W \sigma_1 u_1$. By the argument in the previous paragraph, since $x \in B_W$, we have $\rho + Mx \in B_W$, implying that
\begin{align*}
W^2 \geq \nm{M x + \rho}_2^2
&= \nm{ W \sigma_1 u_1 + (\rho\tr u_1) u_1 + (\rho - (\rho\tr u_1)u_1) }_2^2
\intertext{Since $u_1$ and $(\rho - (\rho\tr u_1)u_1)$ are orthogonal, by Pythagoras, we have }
&= \abs{ W \sigma_1 + \rho\tr u_1 }^2 + \nm{ (\rho - (\rho\tr u_1)u_1) }_2^2 
\\&= ( W \sigma_1)^2 + 2 W \sigma_1 \rho\tr u_1 + (\rho\tr u_1)^2 + \nm{ (\rho - (\rho\tr u_1)u_1) }_2^2
\\&= ( W \sigma_1)^2 + 2 W \sigma_1 \rho\tr u_1 + \nm{(\rho\tr u_1) u_1}_2^2 + \nm{ (\rho - (\rho\tr u_1)u_1) }_2^2
\intertext{By Pythagoras, }
&= ( W \sigma_1)^2 + 2 W \sigma_1 \rho\tr u_1 + \nm{\rho}_2^2.
\end{align*}
Hence, we get the following inequality:
\begin{align*}
W^2 \sigma_1^2 + 2  W (\rho\tr u_1) \sigma_1 + \|\rho^2\|_2 - W^2 \leq 0.
\end{align*}
Solve for $\sigma_1$ and using the fact that $\rho\tr u_1 \geq 0$, we have that,
\begin{align*}
\sigma_1 & \leq \frac{  - \rho\tr u_1  + \sqrt{  (\rho\tr u_1)^2  + (W^2 - \|\rho\|^2)}}{W}  \leq \frac{ \sqrt{ W^2 - \|\rho\|^2 }  }{  W } \leq \sqrt{ 1 - \frac{\|\rho\|^2}{ W^2} }.
\end{align*}

We finally show that $\| \rho \|_2 < W$ unless $M = 0$.
We prove this by contradiction. Assume $\|\rho\|_2 \geq W$.
Following the above argument, take any $w_1 \in B_W$, we must have $w_2 := \rho + M\tr w_1 \in B_W$.  We discuss two cases. 

First if $\rho \not\in \text{range}( {M}\tr)$. In this case, we must have $\|w_2 \|_2^2 = \| \rho \|^2_2 + \|M\tr w_1 \|_2^2 = W^2 + \| M\tr w_1 \|_2^2 > W^2$, as long as $M$ has non-zero entries.  Thus, this case leads to contradiction .

Second, if $\rho \in \text{range}(  {M}\tr   )$. In this case, there must exist a vector $x \neq 0$ such that $M\tr x = \rho$. Consider $\bar x:= W\frac{x}{ \|x\|_2} \in B_W$. We have $w_2 := \rho + M\tr \bar x = \rho \prns{1 + \frac{W}{\|x\|_2} }$, which means that $\|w_2\|_2 > W$, which causes contradiction again.  

So unless $M = 0$, which only happens when $\gamma = 0$ (i.e. horizon is 1), we have $\|\rho\|_2 < W$.
\end{proof}

We now show an approximate version to the equivalence of \cref{prop:equivalent_lbc}. First, recall the bilevel loss from \cref{eq:bilevel_deter}. We use $\mL_{lbc}$ and $\wh\mL_{lbc}$ to denote the population and empirical versions as follows,
\begin{align}
	&\mL_{lbc}(\phi) = \min_{(\rho,M)\in\Theta} \mathbb{E}_{\nu} \left\| \begin{bmatrix}
		M \\ \rho\tr
	\end{bmatrix}\phi(s,a) -  \begin{bmatrix} \gamma \Eb[s' \sim P(s,a)]{\phi(s', \pi_e)} \\ r(s,a) \end{bmatrix}\right\|^2_2
	\\&\wh\mL_{lbc}(\phi) = \min_{(\rho,M)\in\Theta}  \mathbb{E}_{\mD} \left\| \begin{bmatrix}
		M \\ \rho\tr
	\end{bmatrix}\phi(s,a) -  \begin{bmatrix} \gamma\phi(s', \pi_e) \\ r(s,a) \end{bmatrix}\right\|^2_2
	- \min_{g\in\mathcal{G}} \mathbb{E}_{\mD} \left\| g(s,a) - \gamma\phi(s',\pi_e) \right\|_2^2, \label{eq:bilevel-losses-notation}
\end{align}
where $\Theta = \braces{(\rho,M)\in B_W\times\RR^{d\times d}: \|\rho\|\leq\|\rho^\star\|, \|M\|_2\leq\|M^\star\|_2}$ and $(\rho^\star,M^\star)$ are corresponding to the linear BC $\phi^\star$.

\begin{lemma}\label{lm:equivalent_approximate_lbc}
Suppose a feature $\wh\phi$ satisfies $\mL_{lbc}(\wh\phi) \leq \eps^2$. Then, $\wh\phi$ is $\eps(1+W)$-approximately Linear BC, provided $W \geq \frac{\nm{\rho^\star}_2}{1-\nm{M\opt}_2}$.
\end{lemma}
\begin{proof}
Suppose $\mL_{lbc}(\wh\phi) \leq \eps^2$, so there exists $\wh M$ (s.t. $\|\wh M\|_2 \leq \Mub < 1$) and $\wh\rho \in B_W$ s.t.
\begin{align*}
	\EE_{s,a\sim \nu} \left\| \begin{bmatrix} \wh M \\ \wh\rho\tr \end{bmatrix} \phi(s,a)  -  \begin{bmatrix} \gamma \EE_{s'\sim P(s,a)} \phi(s',\pi_e) \\ r(s,a)  \end{bmatrix} \right\|_2^2 \leq \eps^2
\end{align*}
For any $w_1 \in B_W$, we can take $w_2 = \wh\rho + \wh M\tr w_1$. Then, $\|w_2\|_2 \leq \|\wh\rho\|_2 + \|\wh M\|_2 W \leq \|\rho^\star\|_2 + \|M^\star\|_2 W \leq W$ by our assumption on $W$. Hence, 
\begin{align*}
&\max_{w_1 \in B_W} \min_{w_2 \in B_W} \nm{ w_2\tr \phi(s,a) - r(s,a) - \gamma \Eb[s'\sim P(s,a)]{ w_1\tr \phi(s',\pi_e) } }_\nu
\\&\leq \max_{w_1 \in B_W} \nm{ (\wh\rho + \wh M\tr w_1)\tr \phi(s,a) - r(s,a) - \gamma \Eb[s'\sim P(s,a)]{ w_1\tr \phi(s',\pi_e) } }_\nu
\\&= \max_{w_1 \in B_W} \sqrt{ \Eb[s,a \sim \nu]{ \prns{ (\wh\rho + \wh M\tr w_1)\tr \phi(s,a) - r(s,a) - \gamma \Eb[s'\sim P(s,a)]{ w_1\tr \phi(s',\pi_e) } }^2 } }
\\&\leq \max_{w_1 \in B_W} \sqrt{ \Eb[s,a \sim \nu]{ \prns{\wh\rho\tr \phi(s,a) - r(s,a) }^2 } } + \sqrt{ \Eb[s,a\sim\nu]{ \prns{ w_1\tr (\wh M \phi(s,a) - \gamma \Eb[s'\sim P(s,a)]{ \phi(s',\pi_e) }) }^2 } } 
\\&\leq  \eps(1+W),
\end{align*}
as desired.

\end{proof}

\section{Proofs for Representation Learning}

To simplify analysis, assume that the functions in $\mG$ have bounded $\ell_2$ norm, i.e. $\forall g \in \mG, s, a \in \mS \times \mA, \|g(s,a)\|_2 \leq \gamma$. This is reasonable, and can always be achieved by clipping without loss of accuracy, since the target for $g(s,a)$ is $\gamma \Eb[s'\sim P(s,a)]{\phi(s',\pi_e)}$ and $\|\phi(s,a)\|_2 \leq 1$ for any $s,a\in\mS\times\mA$.

\subsection{Lemmas}
Recall that $\lambda_k(A)$ denotes the $k$-th largest eigenvalue of a matrix $A$, i.e. $\lambda_1(A),\lambda_n(A)$ give the largest and smallest eigenvalues respectively. 

\begin{lemma}[Weyl's Perturbation Theorem]\label{lm:weyl-perturbation}
	Let $A, B \in \CC^{n \times n}$ be Hermitian matrices. Then
	\begin{align*}
		\max_k \abs{ \lambda_k(A) - \lambda_k(B) } \leq \nm{A - B}_2.
	\end{align*}
\end{lemma}
\begin{proof}
    Please see \citep[Corollary III.2.6]{bhatia2013matrix}.
\end{proof}

We extend this to be uniform over all $\Phi$.
\begin{lemma}[Uniform spectrum concentration]\label{lm:uniform-spectrum-concentrate}
For any $\delta \in (0, 1)$, w.p. $1-\delta$,
\begin{align*}
\sup_{\phi \in \Phi, k \in [d]} \abs{ \lambda_k( \Sigma(\phi) ) - \lambda_k( \wh\Sigma(\phi) ) } \leq \frac{1}{\sqrt{N}} \prns{ 96 \kappa(\Phi) + 4d + 4\log^{1/2}(1/\delta) }
\end{align*}
\end{lemma}
\begin{proof}
First observe that,
\begin{align*}
    \sup_{\phi \in \Phi} \sup_{k} \abs{ \lambda_k( \Sigma(\phi) ) - \lambda_k( \wh\Sigma(\phi) ) }
    &\leq \sup_{\phi \in \Phi} \nm{ \Sigma(\phi) - \wh\Sigma(\phi) }_2
    \\&= \sup_{\phi \in \Phi, \nm{x}_2 \leq 1} x\tr (\Sigma(\phi) - \wh\Sigma(\phi)) x
    \\&= \sup_{\phi \in \Phi, \nm{x}_2 \leq 1} (\EE_\nu - \EE_\mD) (x\tr \phi(s,a))^2
\end{align*}
Now we want to bound the Rademacher complexity of the class
\begin{align*}
    \mF = \braces{ (s,a) \mapsto (x\tr \phi(s,a))^2, \phi \in \Phi, \nm{x}_2 \leq 1 }
\end{align*}
First, to bound the envelope, we have $(x\tr \phi(s,a))^2 \leq 1$.
To cover, consider any $\phi \in \Phi$ and $x \in \RR^d$ s.t. $\nm{x}_2 \leq 1$. Pick $\wt\phi, \wt x$ close to $\phi, x$, so that 
\begin{align*}
    &\sqrt{ \frac{1}{N} \sum_{i=1}^N \prns{ (x\tr \phi(s_i,a_i))^2 - (\wt x\tr \wt \phi(s_i,a_i))^2 }^2 }
    \\&\leq 2 \sqrt{ \frac{1}{N} \sum_{i=1}^N \prns{ (x\tr \phi(s_i,a_i) - \wt x\tr \wt \phi(s_i,a_i)) }^2 }
    \\&\leq 2 \prns{ \sqrt{ \frac{1}{N} \sum_{i=1}^N \prns{ x\tr (\phi(s_i,a_i) - \wt \phi(s_i,a_i)) }^2 } + \sqrt{ \frac{1}{N} \sum_{i=1}^N \prns{ (x-\wt x)\tr \wt \phi(s_i,a_i) }^2 } }
    \\&\leq 2 \prns{ d_\Phi(\phi,\wt\phi) + \nm{x-\wt x}_2 }
\end{align*}
So it suffices to take $d_\Phi(\phi,\wt\phi), \nm{x-\wt x}_2 \leq t/4$ to $t$-cover $\mF$ in $L_2(\mD)$. Note the $t/4$-covering number for $x$ in the unit ball is $(1+8/t)^d$. Thus, by Dudley's entropy bound ((5.48) of \citep{wainwright2019high}),
\begin{align*}
    \mR_N(\mF) 
    &\leq \frac{24}{\sqrt{N}} \int_0^1 \log^{1/2}( \mN(t/4, \Phi) \cdot ( 1+8/t )^d ) \diff t
    \\&\leq \frac{96}{\sqrt{N}} \prns{ \kappa(\Phi) + 4\sqrt{d} }
\end{align*}
Thus, by Theorem 4.10 of \citep{wainwright2019high}, w.p. $1-\delta$,
\begin{align*}
    \sup_{\phi \in \Phi, \nm{x}_2 \leq 1} (\EE_\nu - \EE_\mD) (x\tr \phi(s,a))^2 
    &\leq 2\mR_N(\mF) + \frac{4\log^{1/2}(1/\delta)}{\sqrt{N}} \\
    &\leq \frac{1}{\sqrt{N}} \prns{ 96 \kappa(\Phi) + 4\sqrt{d} + 4\log^{1/2}(1/\delta) }
\end{align*}
\end{proof}

We now prove the double sampling lemma, i.e. modified Bellman Residual Minimization \citep{chen2019information}. This will help deal with the double sampling issue when transitions are stochastic. 
Recall that $\mG$ is a function class of functions $g: \mX \mapsto \RR^d$.
Let $\nu$ be a distribution over $\mX \subset \RR^d$ and, for any $x \in \mX$, let $P(x)$ be a distribution over $\mY \subset \RR^d$.
\begin{lemma}[Double Sampling]\label{lm:double-sampling}
Suppose $x \mapsto \Eb[y \sim P(x)]{y} \in \mG$. Then,
\begin{align*}
    \Eb[x \sim \nu]{\nm{x - \Eb[y \sim P(x)]{y}}_2^2} = \Eb[x \sim \nu, y \sim P(x)]{\nm{x - y}_2^2} - \inf_{g \in \mG} \Eb[x \sim \nu, y \sim P(x)]{\nm{g(x) - y}_2^2}
\end{align*}
\end{lemma}
\begin{proof}
\begin{align*}
&\Eb[x \sim \nu, y \sim P(x)]{\nm{x - y}_2^2} - \Eb[x \sim \nu]{\nm{x - \Eb[y \sim P(x)]{y}}_2^2}
\\&= \Eb[x \sim \nu]{ \Eb[y \sim P(x)]{\nm{x}_2^2 - 2\ip{x}{y} + \nm{y}_2^2} - \prns{ \nm{x}_2^2 - 2\ip{x}{\Eb[y\sim P(x)]{y}} + \nm{\Eb[y\sim P(x)]{y}}_2^2 } } 
\\&= \Eb[x\sim \nu]{ \Eb[y\sim P(x)]{\nm{y}_2^2} - \nm{\Eb[y \sim P(x)]{y}}_2^2 }
\\&= \Eb[x\sim \nu]{ \Eb[y \sim P(x)]{ \nm{y - \Eb[y \sim P(x)]{y}}_2^2 } }
\end{align*}
where the last step uses the fact that $\nm{\Eb[y \sim P(x)]{y}}_2^2 = \Eb[y \sim P(x)]{\ip{y}{\Eb[y\sim P(x)]{y}}}$, and completing the square.
Now, observe that, assuming $g\opt(x) \mapsto \Eb[y \sim P(x)]{y} \in \mG$, we have that it is the minimizer of,
\begin{align}
    g\opt \in \argmin_{g\in\mG} \Eb[x \sim \nu, y \sim P(x)]{\nm{g(x) - y}_2^2}
\end{align}
which completes the proof.
\end{proof}

\subsection{Concentration lemmas}

For any $\phi$, define the optimal $\rho, M, g$ for our losses as follows:
\begin{align*}
    &\rho_\phi \in \argmin_{(\rho,\_)\in\Theta} \Eb[\nu]{ (\rho\tr \phi(s,a) - r(s,a))^2 }
    \\&M_\phi \in \argmin_{(\_,M)\in\Theta} \Eb[\nu\circ P]{ \|M\phi(s,a) - \gamma\phi(s',\pi_e)\|_2^2 }
    \\&g_\phi \in \argmin_{g \in \mG} \Eb[\nu\circ P]{ \|g(s,a) - \gamma\phi(s',\pi)\|_2^2 }
\end{align*}
Similarly, define $\wh\rho_\phi, \wh M_\phi, \wh g_\phi$ to be minimizers of the above losses when expectation is taken over the empirical distribution $\mD$, instead of the population distribution $\nu$.
Observe that the unconstrained minimization yields a closed form solution for $g_\phi$ as $g_{\phi}(s,a) = \gamma\Eb[s'\sim P(s,a)]{\phi(s',\pi)}$ -- \cref{asm:g-realizability} posits that $\mG$ is rich enough to capture this.

The key property of our squared losses is that the second moment can be upper bounded by the expectation, which allows us to invoke the second part of the above \cref{lm:modified-bernstein}.
We now combine this with covering to get uniform convergence results.

\begin{lemma}\label{lm:rep-learning-fast-rates}
For any $\delta\in(0,1)$, w.p. at least $1-\delta$, for any $\phi\in\Phi,\rho\in B_W$ and $M\in\RR^{d\times d}$ with $\|M\|_2 \leq 1$, we have
\begin{align*}
    &\abs{ \Eb[\mD]{ (\rho\tr\phi(s,a) - r(s,a))^2 }  -  \Eb[\nu]{ (\rho\tr\phi(s,a) - r(s,a))^2 } } \leq \frac{1}{2} \Eb[\nu]{ (\rho\tr\phi(s,a) - r(s,a))^2 } + \eps_\rho,
    \\&\abs{ \Eb[\mD]{ \| M\phi(s,a) - \gamma\phi(s',\pi) \|_2^2 - \|g_\phi(s,a) - \gamma\phi(s',\pi) \|_2^2 } - \Eb[\nu\circ P]{ \| M\phi(s,a) - \gamma\phi(s',\pi) \|_2^2 - \|g_\phi(s,a) - \gamma\phi(s',\pi) \|_2^2 } } 
	\\& \leq \frac{1}{2} \Eb[\nu\circ P]{ \| M\phi(s,a) - \gamma\phi(s',\pi) \|_2^2 - \|g_\phi(s,a) - \gamma\phi(s',\pi) \|_2^2 } + \eps_M,
\end{align*}
and, assuming realizability (\cref{asm:g-realizability}), for every $g \in \mG$, we have
\begin{align*}
    &\abs{ \Eb[\mD]{ \|g(s,a) - \gamma\phi\opt(s',\pi)\|_2^2 - \|g_{\phi\opt}(s,a) - \gamma\phi\opt(s',\pi)\|_2^2 } - \Eb[\nu]{\|g(s,a)-g_{\phi\opt}(s,a)\|_2^2} }
	\\&\leq \frac{1}{2} \Eb[\nu]{\|g(s,a)-g_{\phi\opt}(s,a)\|_2^2}  + \eps_g,
\end{align*}
where $\eps_\rho,\eps_M,\rho_g$ are defined below.
\end{lemma}

\paragraph{Finite function classes} 
Assuming $\Phi$ and $\mG$ are finite, we have
\begin{align*}
    &\eps_\rho \leq \frac{ 6d(1+W)^2 \log(4W \abs{\Phi} N /\delta)}{N}
    \\&\eps_M \leq \frac{32 d^2 \log(4 \abs{\Phi} N/\delta)}{N}
    \\&\eps_g \leq \frac{20\gamma^2\log(2\abs{\mG}/\delta)}{N}.
\end{align*}

\begin{proof}[Proof for $\eps_\rho$]
    For a fixed $\phi \in \Phi, \rho \in B_W$, apply \cref{lm:modified-bernstein} to $X_i = (\rho\tr\phi(s_i,a_i) - r(s_i,a_i))^2$.
	The envelope is $\abs{X_i} \leq (1+W)^2$ and 
	the second moment is bounded $\Eb{X_i^2} \leq (1+W)^2 \Eb{X_i}$.
	So, the error from the lemma is $\frac{ 2(1+W)^2 \log(2/\delta)}{N}$.
	Now union bound over an $\eps$-net of $B_W$.
	Since $\abs{ (\rho\tr\phi(s,a) - r(s,a))^2 - (\wt\rho\tr\phi(s,a) - r(s,a))^2 } = \abs{ (\rho-\wt\rho)\tr\phi(s,a) ((\rho+\wt\rho)\tr\phi(s,a) + 2r(s,a)) } \leq 2(1+W) \| \wt\rho - \rho \|_2$, we consider a $\frac{1}{N}$-net of $B_W$ which requires $\prns{1 + 2WN}^d$ points. 
	The error from this $\eps$-net approximation is at most $\frac{4(1+W)}{N}$.
	Finally, union bound over $\Phi$.
\end{proof}
\begin{proof}[Proof for $\eps_M$]
    For a fixed $\phi \in \Phi$ and $M \in \RR^{d \times d}$ s.t. $\|M\|_2 < 1$, apply \cref{lm:modified-bernstein} to $X_i = \|a\|_2^2 - \|b\|_2^2$ where $a = M\phi(s_i, a_i) - \gamma\phi(s_i', \pi), b = g_\phi(s_i,a_i) - \gamma\phi(s_i', \pi)$.
	The envelope is $\abs{X_i} = \abs{X_i} \leq \|a\|_2^2 + \|b\|_2^2 \leq (1 + \gamma)^2 + (2\gamma)^2 \leq 8$. Further, observe that $\|a\|_2^2-\|b\|_2^2 = \langle a+b,a-b\rangle \leq \|a+b\|_2\|a-b\|_2$. So, the second moment is bounded $\Eb{X_i^2} \leq \Eb{\|a+b\|_2^2 \|a-b\|_2^2} \leq (1+3\gamma)^2 \Eb{\|M\phi(s_i, a_i) - \gamma g_\phi(s_i, a_i)\|_2^2} \leq 16 \Eb{X_i}$, where we used \cref{lm:double-sampling} to give us
	\begin{align*}
	   \Eb{\|M\phi(s_i, a_i) - \gamma g_\phi(s_i, a_i)\|_2^2} = \Eb{\|M\phi(s_i, a_i) - \gamma\phi(s_i', \pi)\|_2^2 - \|g_\phi(s_i,a_i) - \gamma\phi(s_i', \pi)\|_2^2 }.
	\end{align*}
	So, the error from the lemma is $\frac{24\log(2/\delta)}{N}$. 
	Now union bound over an $\eps$-net of $\braces{M\in\RR^{d\times d}: \|M\|_2 \leq 1}$.
	Observe that
	\begin{align*}
		&\abs{ \prns{ \| M\phi(s,a) - \gamma\phi(s',\pi) \|_2^2 - \|g_\phi(s,a) - \gamma\phi(s',\pi) \|_2^2 } - \prns{ \| \wt M\phi(s,a) - \gamma\phi(s',\pi) \|_2^2 - \|g_{\phi}(s,a) - \gamma\phi(s',\pi) \|_2^2 } }
		\\&\leq \nm{(M-\wt M)\phi(s_i,a_i)}_2 \nm{ M\phi(s_i,a_i) + \wt M\phi(s_i,a_i) - 2\gamma\phi(s_i', \pi) }_2
		\\&\leq \nm{(M-\wt M)}_2 \cdot (2+2\gamma) \leq 4\nm{(M-\wt M)}_2.
	\end{align*}
	Consider a $\frac{1}{N}$-net (under $\|\|_F$) for $\{M \in \RR^{d \times d}: \|M\|_F \leq \sqrt{d} \}$, which requires $(1+2N\sqrt{d})^{d^2}$ points since it is like the $\ell_2$ for a $d^2$-dimensional vector.
	This is a $\frac{1}{N}$-net (under $\|\|_2$) for the subset $\braces{M\in\RR^{d\times d}: \|M\|_2 \leq 1}$ since $\|M\|_2 \leq \|M\|_F \leq \sqrt{d} \|M\|_2$. 
	The error from this $\eps$-net approximation is at most $\frac{8}{N}$.
	Finally, union bound over $\Phi$.
\end{proof}
\begin{proof}[Proof for $\eps_g$]
    For a fixed $g \in \mG$, apply \cref{lm:modified-bernstein} to $X_i = \|g(s_i,a_i)-\gamma\phi\opt(s_i',\pi)\|_2^2 - \|g_{\phi\opt}(s,a) - \gamma\phi\opt(s',\pi)\|_2^2$, the excess regression loss.
	Under realizability \cref{asm:g-realizability}, we have $\Eb{X_i} = \nm{g - g_{\phi\opt}}_2^2$, since
	\begin{align*}
		\Eb{X_i} 
		&= \Eb{\nm{ g(s,a) -  \gamma\phi\opt(s',\pi) }_2^2 - \nm{ g_{\phi\opt}(s,a) - \gamma\phi\opt(s',\pi) }_2^2}
		\\&= \Eb{ \nm{g(s,a) - g_{\phi\opt}(s,a)}_2^2 + 2\langle g_{\phi\opt}(s,a) - \gamma\phi\opt(s',\pi), g(s,a) - g_{\phi\opt}(s,a) \rangle  }
		\intertext{By definition of $g_{\phi^\star}$, we have $\Eb[s' \sim P(s,a)]{g_{\phi\opt}(s,a) - \gamma\phi\opt(s',\pi)} = 0$, so, }
		&= \Eb{ \nm{g(s,a) - g_{\phi\opt}(s,a)}_2^2 }
	\end{align*}
	The envelope is $\abs{X_i} \leq (2\gamma)^2$ and the second moment is bounded $\Eb{X_i^2} \leq \Eb{\|g(s,a) + g_{\phi\opt}(s,a) - 2\gamma\Eb[a' \sim \pi(s')]{\phi\opt(s',a')}\|_2^2 \|g(s,a) - g_{\phi\opt}(s,a)\|_2^2 } \leq (4\gamma)^2 \Eb{X_i}$.
	So the error term from the lemma is $\frac{20\gamma^2\log(2/\delta)}{N}$.
	Now union bound over $\mG$.
\end{proof}

\paragraph{Infinite function classes}
When $\Phi$ and $\mG$ are infinite, we need to assume some metric entropy conditions. Then in the final step of the finite-class proofs above, we union bound on a well-chosen $\eps$-net and collect an additional approximation error which is on the order of $\mO(1/N)$.

\begin{assumption}\label{asm:entropy}
For $\mF\in\braces{\Phi,\mG}$, we assume there exists $p\in\RR_{++}$ such that $\mN(t,\mF) \lesssim t^{-p}$, where the net is under the following distances,
\begin{align*}
    d_\Phi(\phi, \wt \phi) &\defeq \Eb[\mD]{\nm{ \phi(s,a) - \wt\phi(s,a) }_2} + \Eb[\nu]{\nm{ \phi(s,a) - \wt\phi(s,a) }_2} 
    \\&+ \Eb[\mD]{\nm{ \phi(s',\pi) - \wt\phi(s',\pi) }_2} + \Eb[s,a \sim \mD, s' \sim P(s,a)]{\nm{ \phi(s',\pi) - \wt\phi(s',\pi) }_2} 
    \\&+ \Eb[s,a \sim \nu, s' \sim P(s,a)]{\nm{ \phi(s',\pi) - \wt\phi(s',\pi) }_2}
    \\ d_\mG(g, \wt g) &\defeq \sqrt{ \Eb[\mD]{ \nm{ g(s_i,a_i) - \wt g(s_i,a_i) }_2^2 } } + \sqrt{ \Eb[\nu]{ \nm{ g(s,a) - \wt g(s,a) }_2^2 } }
\end{align*}
\end{assumption}
Note that this assumption is automatically satisfied for any $p>0$ by VC classes \citep[Theorem 2.6.4]{vandervaart1996weak}.
Under this assumption, we have
\begin{align*}
    &\eps_\rho \lesssim \frac{ d(1+W)^2 (1+p) \log(W N /\delta)}{N}
    \\&\eps_M \lesssim \frac{ d^2 (1+p) \log(N/\delta)}{N}
    \\&\eps_g \lesssim \frac{\gamma^2p\log(N/\delta)}{N}.
\end{align*}

\begin{proof}[Proof for $\eps_\rho$]
From before, we showed for a fixed $\phi\in\Phi$, we have $\eps_\rho \lesssim \frac{d(1+W)^2\log(WN/\delta)}{N}$. Observe that $\abs{ (\rho\tr\phi(s,a)-r(s,a))^2 - (\rho\tr\wt\phi(s,a)-r(s,a))^2 }= \abs{ \rho\tr(\phi(s,a)-\wt\phi(s,a)) (\rho\tr(\phi(s,a)+\wt\phi(s,a)) + 2r(s,a)) } \leq 2W(1+W) \|\phi(s,a)-\wt\phi(s,a)\|_2$, and so the difference of the loss with $\phi$ and the loss with $\wt\phi$ is bounded by $2W(1+W)d_\Phi(\phi,\wt\phi)$. Now union bound over a $\frac{1}{N}$-net, which requires $\mO(N^p)$ points by \cref{asm:entropy}. The approximation error using the net is $\frac{2W(1+W)}{N}$.
\end{proof}
\begin{proof}[Proof for $\eps_M$]
From before, we showed for a fixed $\phi\in\Phi$, we have $\eps_M \lesssim \frac{d^2\log(N/\delta)}{N}$. Observe that
\begin{align*}
    &\abs{ \prns{ \| M\phi(s,a) - \gamma\phi(s',\pi) \|_2^2 - \|g_\phi(s,a) - \gamma\phi(s',\pi) \|_2^2 } - \prns{ \| M\wt\phi(s,a) - \gamma\wt\phi(s',\pi) \|_2^2 - \|g_{\wt\phi}(s,a) - \gamma\wt\phi(s',\pi) \|_2^2 } }
	\\&\leq \|M(\phi(s,a)-\wt\phi(s,a))-\gamma(\phi(s',\pi)-\wt\phi(s',\pi))\| \|M(\phi(s,a)+\wt\phi(s,a)) - \gamma(\phi(s',\pi)-\wt\phi(s',\pi))\|
	\\&+ \|g_\phi(s,a) - g_{\wt\phi}(s,a) - \gamma(\phi(s',\pi)-\wt\phi(s',\pi))\| \|g_\phi(s,a)+g_{\wt\phi}(s,a) - \gamma(\phi(s',\pi)+\wt\phi(s',\pi)) \|
	\\&\leq \prns{ \|\phi(s,a)-\wt\phi(s,a)\| + \gamma \|\phi(s',\pi)-\wt\phi(s',\pi)\| } \cdot (2+2\gamma) 
	\\&+ \prns{ \|g_\phi(s,a)-g_{\wt\phi}(s,a)\| + \gamma \|\phi(s',\pi)-\wt\phi(s',\pi)\| } \cdot (2+2\gamma)
	\intertext{Using closed form solution for $g_\phi$, }
	&\leq 16 d_\Phi(\phi,\wt\phi).
\end{align*}
Now union bound over a $\frac{1}{N}$-net, which requires $\mO(N^p)$ points by \cref{asm:entropy}. The approximation error using the net is at most $\frac{16}{N}$.
\end{proof}
\begin{proof}[Proof for $\eps_g$]
From before, we showed for a fixed $g\in\mG$, we have $\eps_g \lesssim \frac{\gamma^2\log(1/\delta)}{N}$. Observe that
\begin{align*}
    &\abs{ \|g(s,a)-\gamma\phi^\star(s',\pi)\|^2 - \|\wt g(s,a)-\gamma\phi^\star(s',\pi)\|^2 }
    \\&\leq \|g(s,a)-\wt g(s,a)\|\|g(s,a)+\wt g(s,a) - 2\gamma\phi^\star(s',\pi)\|
    \\&\leq (2+2\gamma) d_\mG(g,\wt g).
\end{align*}
Now union bound over a $\frac{1}{N}$-net, which requires $\mO(N^p)$ points by \cref{asm:entropy}. The approximation error using the net is at most $\frac{4}{N}$. 
\end{proof}

\subsection{Main Results}

\begin{lemma}\label{lm:fast-rates-rep-learning-bound}
    Suppose $\phi\opt \in \Phi$ is Linear BC.
	Suppose \cref{asm:g-realizability} if transitions are stochastic. 
	Moreover, suppose $\phi\opt$ is feasible in the bilevel optimization (and so $\wh\mL_{lbc}(\wh\phi) \leq \wh\mL_{lbc}(\phi\opt)$).
	Then, w.p. $1-5\delta$,
	\begin{align*}
	    \mL_{lbc}(\wh\phi) \leq \frac{24d(1+W)^2\log(4W|\Phi| N/\delta) + 128 d^2 \log(4|\Phi|N/\delta) + 40\gamma^2 \log(2|\mG|/\delta)}{N}
	\end{align*}
\end{lemma}
\begin{proof}
\begin{align*}
	&\mL_{lbc}(\wh\phi)
	\\&= \Eb[\nu]{ (\rho_{\wh\phi}\tr \wh\phi(s,a) - r(s,a) )^2 } + \Eb[\nu \circ P]{ \|M_{\wh\phi} \wh\phi(s,a) - \gamma\wh \phi(s',\pi)\|_2^2 } - \Eb[\nu\circ P]{ \|g_{\wh\phi}(s,a) - \gamma\wh\phi(s',\pi)\|_2^2 } 
	\intertext{Since $\rho_{\wh\phi}, M_{\wh\phi}$ are minimizers under $\nu$, }
	&\leq \Eb[\nu]{ (\wh\rho_{\wh\phi}\tr \wh\phi(s,a) - r(s,a) )^2 } + \Eb[\nu \circ P]{ \|\wh M_{\wh\phi} \wh\phi(s,a) - \gamma\wh \phi(s',\pi)\|_2^2 } - \Eb[\nu\circ P]{ \|g_{\wh\phi}(s,a) - \gamma\wh\phi(s',\pi)\|_2^2 } 
	\intertext{By the $\rho,M$ parts of \cref{lm:rep-learning-fast-rates}, }
	&\leq 2\Eb[\mD]{ (\wh\rho_{\wh\phi}\tr \wh\phi(s,a) - r(s,a) )^2 } + 2 \Eb[\mD]{ \|\wh M_{\wh\phi} \wh\phi(s,a) - \gamma\wh\phi(s',\pi)\|_2^2 } - 2 \Eb[\mD]{ \|g_{\wh\phi}(s,a) - \gamma\wh\phi(s',\pi)\|_2^2 }  + 2\eps_\rho + 2\eps_M
	\intertext{Since $\wh g_{\wh\phi}$ minimizes under $\mD$, }
	&\leq 2\Eb[\mD]{ (\wh\rho_{\wh\phi}\tr \wh\phi(s,a) - r(s,a) )^2 } + 2 \Eb[\mD]{ \|\wh M_{\wh\phi} \wh\phi(s,a) - \gamma\wh\phi(s',\pi)\|_2^2 } - 2 \Eb[\mD]{ \|\wh g_{\wh\phi}(s,a) - \gamma\wh\phi(s',\pi)\|_2^2 }  + 2\eps_\rho + 2\eps_M
	\intertext{By optimality of $\wh\phi$ under $\wh\mL_{lbc}$, }
	&\leq 2\Eb[\mD]{ (\wh\rho_{\phi\opt}\tr \phi\opt(s,a) - r(s,a) )^2 } + 2 \Eb[\mD]{ \|\wh M_{\phi\opt} \phi\opt(s,a) - \gamma\phi\opt(s',\pi)\|_2^2 } - 2 \Eb[\mD]{ \|\wh g_{\phi\opt}(s,a) - \gamma\phi\opt(s',\pi)\|_2^2 }  
	\\&+ 2\eps_\rho + 2\eps_M
	\intertext{By the $\mG$ part of \cref{lm:rep-learning-fast-rates}, we have $\Eb[\mD]{ \|g_{\phi\opt}(s,a) - \phi\opt(s',\pi)\|_2^2 - \|\wh g_{\phi\opt}(s,a) - \phi\opt(s',\pi)\|_2^2 } \leq -\frac{1}{2}\Eb[\nu]{ \|g_{\phi\opt} - \wh g_{\phi\opt}\|_2^2} + \eps_g \leq \eps_g$. Then, using the optimality of $\wh\rho$ and $\wh M$ under $\mD$, }
	&\leq 
	2\Eb[\mD]{ (\rho_{\phi\opt}\tr \phi\opt(s,a) - r(s,a) )^2 } + 2 \Eb[\mD]{ \|M_{\phi\opt} \phi\opt(s,a) - \phi\opt(s',\pi)\|_2^2 } - 2 \Eb[\mD]{ \|g_{\phi\opt}(s,a) - \phi\opt(s',\pi)\|_2^2 } 
	\\&+ 2\eps_\rho + 2\eps_M + 2\eps_g
	\intertext{By the $\rho,M$ parts of \cref{lm:rep-learning-fast-rates}, }
	&\leq  3\Eb[\nu]{ (\rho_{\phi\opt}\tr\phi\opt(s,a) - r(s,a))^2 } + 3\Eb[\nu\circ P]{ \|M_{\phi\opt}\phi\opt(s,a) - \phi\opt(s',\pi) \|_2^2 } - 3\Eb[\nu]{ \|g_{\phi\opt}(s,a) - \phi\opt(s',\pi) \|_2^2 } 
	\\&+ 4\eps_\rho + 4\eps_M + 2\eps_g
	\intertext{By \cref{asm:g-realizability} and \cref{lm:double-sampling}, }
	&= 3\mL_{lbc}(\phi^\star) + 4\eps_\rho + 4\eps_M + 2\eps_g
	\intertext{By assumption that $\phi\opt$ is Linear BC and \cref{prop:equivalent_lbc}, } %
	&= 4\eps_\rho + 4\eps_M + 2\eps_g.
\end{align*}

\end{proof}

We now prove \cref{thm:learning-phi}, in the general stochastic case. The deterministic transitions case is subsumed by ignoring the minimization over $\mG$, i.e. setting the complexity term of $\mG$ to zero.

\learningphi*
\begin{proof}[Proof of \cref{thm:learning-phi}]
	First, by our assumption that $\sqrt{N} \geq 4(96\kappa(\Phi) + 4\sqrt{d} + 4\log^{1/2}(1/\delta))/\beta$, \cref{lm:uniform-spectrum-concentrate} implies that w.p. at least $1-\delta$, we have $\sup_{\phi \in \Phi} \abs{ \lambda_{min}(\Sigma(\phi)) - \lambda_{min}(\wh\Sigma(\phi)) } \leq \beta/4$. Under this high probability event, we have two important consequences:
	\begin{enumerate}
		\item $\phi\opt$ is feasible in \cref{eq:phihat-def}, since $\lambda_{min}(\wh\Sigma(\phi\opt)) \geq \lambda_{min}(\Sigma(\phi\opt)) - \beta/4 \geq \beta(1-1/4) \geq \beta/2$. In particular, this means $\wh\mL_{lbc}(\wh\phi) \leq \wh\mL_{lbc}(\phi^\star)$.
		\item The covariance of $\wh\phi$ has lower-bounded eigenvalues, since $\lambda_{min}(\Sigma(\wh\phi)) \geq \lambda_{min}(\wh\Sigma(\wh\phi)) - \beta/4 \geq \beta(1/2-1/4) \geq \beta/4$.
	\end{enumerate}

	Now, apply \cref{lm:fast-rates-rep-learning-bound} to bound $\mL_{lbc}(\wh\phi)$, so w.p. $1-5\delta$,
	\begin{align*}
	    \mL_{lbc}(\wh\phi) \leq \frac{24d(1+W)^2\log(4W|\Phi| N/\delta) + 128 d^2 \log(4|\Phi|N/\delta) + 40\gamma^2 \log(2|\mG|/\delta)}{N}
	\end{align*}
	
	By \cref{lm:equivalent_approximate_lbc}, we have that $\wh\phi$ is $\wh\eps$-approximately Linear BC, with parameter
	\begin{align*}
		\wh\eps 
		&\leq (1+W) \cdot \sqrt{ \frac{24d(1+W)^2\log(4W|\Phi| N/\delta) + 128 d^2 \log(4|\Phi|N/\delta) + 40\gamma^2 \log(2|\mG|/\delta)}{N} }
		\\&\leq \frac{13d(1+W)^2\log^{1/2}(4W|\Phi|N/\delta)}{\sqrt{N}} + \frac{7\gamma(1+W)\log^{1/2}(2|\mG|/\delta)}{\sqrt{N}}
	\end{align*}
	
	Finally, we remark that the $\wh W$ for $\wh\phi$ (in the approximately Linear BC case) is upper bounded by a polynomial in $W\opt$ in the assumed exact Linear BC of $\phi\opt$. 
	Consider our assumption that $\phi\opt$ is exactly Linear BC with $W\opt = W$ (use $\star$ to highlight that it is the $W$ in the assumption, which we now show matches the $W$ in the result).
	Then, by \cref{prop:equivalent_lbc}, $\exists M\opt$ with $\|M\opt\|_2 \leq \sqrt{ 1-\frac{\|\rho^\star\|_2^2}{W^{\star 2}} }$. Hence, it suffices to minimize over this smaller ball for $\wh M$, so that $\|\wh M\|_2 \leq \|M\opt\|_2$.
	Now, take the smallest possible $W$ in \cref{lm:equivalent_approximate_lbc}, so that 
	\begin{align*}
		\wh W 
		&= \frac{\nm{\rho^\star}_2}{1- \nm{M\opt}_2} 
		\\&\leq \frac{\nm{\rho^\star}_2}{1-\sqrt{1 - \frac{\nm{\rho^\star}_2^2}{W^{\star 2}}}}
		\\&= \frac{\nm{\rho^\star}_2 (1 + \sqrt{1 - \frac{\nm{\rho^\star}_2^2}{W^{\star 2}}}) }{\frac{\nm{\rho^\star}_2^2}{W^{\star 2}}}
		\\&\leq \frac{2W^{\star 2}}{\nm{\rho^\star}_2},
	\end{align*}
	which is a polynomial in $W^\star$.
\end{proof}

Our end-to-end result is deduced by chaining our LSPE theorem and the above theorem together. 
\endtoend*
\begin{proof}[Proof of \cref{thm:end-to-end}]
	We first apply \cref{thm:learning-phi} to see that $\wh\phi$ satisfies the two properties needed for LSPE. It is indeed approximately Linear BC, with $\wh\eps$ specified in the theorem, and also has coverage (i.e. $\lambda_{min}(\Sigma(\wh\phi)) \geq \beta/4$). Using these two facts, and on a separate independent dataset $\mD_2$ (needs to be a separate dataset since $\wh\phi$ is data-dependent), we run LSPE and directly apply \cref{thm:fqe} for the result.
\end{proof}

\section{Implementation Details}
\label{app:implementation}
Here we detail all environment specifications and hyperparameters used in the main text.
\subsection{Dataset Details}
Using the publicly released implementation for DrQ-v2, we trained high quality target policies and saved checkpoints for offline behavior datasets. We refer the readers to \citet{yarats2021mastering} for exact hyperparameters.
\begin{table}[h]
    \centering
    \begin{tabular}{c|cc}
    \toprule
        Task & \shortstack{Target\\Performance} & \shortstack{Behavior\\Performance} \\
    \midrule
        \texttt{Finger Turn Hard} & 927 & 226 (24\%) \\
        \texttt{Cheetah Run}      & 758 & 192 (25\%)\\
        \texttt{Quadruped Walk}   & 873 & 236 (27\%)\\
        \texttt{Humanoid Stand}   & 827 & 277 (33\%)\\
    \bottomrule
    \end{tabular}
    \caption{Performance for target and behavior policies used to collect evaluation and offline datasets respectively.}
    \label{tab:perf}
\end{table}
\subsection{Environment Details}
Following the standards used by DrQ-v2 \cite{yarats2021mastering}, all environments have a maximum horizon length of 500 timesteps. This is achieved by the behavior/target policy having an action repeat of 2 frames. Furthermore, each state is 3 stacked frames that are each 84 $\times$ 84 dimensional RGB images (thus $9\times84\times84$). 
\begin{table}[h]
    \centering
    \begin{tabular}{c|ccc}
    \toprule
        Task & Action Space Dimension & Task Traits & Reward Type\\
    \midrule
        \texttt{Finger Turn Hard} & 2 & turn & sparse\\
        \texttt{Cheetah Run} & 6 & locomotion & dense\\
        \texttt{Quadruped Walk} & 12 & locomotion & dense\\
        \texttt{Humanoid Stand} & 21 & stand & dense\\
    \bottomrule
    \end{tabular}
    \caption{Task descriptions, action space dimension, and reward type for each tested environment.}
    \label{tab:my_label}
\end{table}
\subsection{Representation Architecture and Hyperparameter Details}
We adopt the same network architecture as DrQ-v2's critic, first introduced in SAC-AE \cite{yarats2019sacae}. 
More specifically, to process pixel input, we have a 5 layer ConvNet with $3\times3$ kernels and 32 channels with \texttt{ReLU} activations. The first convolutional layer has a stride of 2 while the rest has stride 1. The output is fed through a single fully connected layer normalized by LayerNorm. Finally, there is a \texttt{tanh} nonlinearity on the outputted 50 dimensional state-representation. The action is then concatenated with this output and fed into a 4-layer MLP all with \texttt{ReLU} activations. 

\begin{table}[h]
    \centering
    \begin{tabular}{c|c}
    \toprule
        Hyperparameter & Value \\
    \midrule
        Feature Dimension & 512 \\
        Weight Initialization & orthogonal init. \\
        Optimizer & Adam \\
        Learning Rate & $1\times 10^{-5}$ \\
        Batch Size & 2048 \\
        Training Epochs & 200 \\
        $\tau$ (target) & 0.005 \\
        $\lambda_{\text{Design}}$ & $5\times 10^{-6}$ \\
    \bottomrule
    \end{tabular}
    \caption{Hyperparameters used for~\alg}
    \label{tab:hyperparam}
\end{table}
\newpage
\subsection{Benchmarks and Metrics}
\textbf{Modifications to CURL: }
Originally, CURL only does contrastive learning between the image states with data augmentation. For OPE, apply the same CURL objective to the state-action feature detailed in the previous section. Note we also train CURL with the same random cropping image augmentations presented by the authors. Finally, since we are not interleaving the representation learning with SAC, we do not have a Q prediction head.

\textbf{Modifications to SPR: }
We use the same image encoder as our features for SPR. The main difference is in the architecture of the projection layers where we implement as 3-layer mlp with \texttt{ReLU} activations. Note that these are additional parameters that neither CURL nor~\alg~require. Finally, similarly to CURL, we do not have an additional Q-prediction head.

\textbf{Spearman Ranking Correlation Metric: } This rank correlation measures the correlation between the ordinal rankings of the value estimates and the ground truth returns. As defined in \citet{fu2021deepope}, we have for policies $1, 2, \ldots, N$, true returns $V_{1:N}$, and estimated returns $\hat{V}_{1:N}$:
\begin{equation*}
    \text{Ranking Correlation} = \frac{\text{Cov}\left(V_{1:N},\hat{V}_{1:N}\right)}{\sigma\left(V_{1:N}\right)\sigma(\hat{V}_{1:N})}
\end{equation*}

\end{document}